\documentclass[twoside]{article}

\usepackage[accepted]{aistats2026}
\usepackage[round]{natbib}

\usepackage[utf8]{inputenc} % allow utf-8 input
\usepackage[T1]{fontenc}    % use 8-bit T1 fonts
\usepackage{hyperref}       % hyperlinks
\usepackage{url}            % simple URL typesetting
\usepackage{booktabs}       % professional-quality tables
\usepackage{amsfonts}       % blackboard math symbols
\usepackage{nicefrac}       % compact symbols for 1/2, etc.
\usepackage{microtype}      % microtypography
\usepackage{xcolor}  
\usepackage{graphicx} % Required for inserting images
\usepackage{enumitem}
\usepackage{placeins}

\usepackage{textcomp} % Additional text symbols
\usepackage{xspace}
\usepackage{comment}

\usepackage[normalem]{ulem}

\usepackage{amsmath}
\usepackage{enumitem}
\usepackage{amsfonts, mathtools,amssymb, amsthm}
\usepackage{thm-restate}
\usepackage{mleftright}

\newtheorem{theorem}{Theorem}[section]
\newtheorem{lemma}[theorem]{Lemma}

\newtheorem{example}{Example}

\usepackage{algorithm}
\usepackage[noend]{algpseudocode}

\usepackage[dvipsnames,table]{xcolor}
\definecolor{myblue}{HTML}{1F77B4}

\usepackage{multirow}
\usepackage{multicol}
\usepackage{chngpage} % allows for temporary adjustment of side margins

\usepackage{mfirstuc}

\usepackage{caption}
\usepackage{subcaption}
\usepackage{centernot}

\usepackage{aliascnt}
\usepackage[numbered]{bookmark}
\usepackage[capitalise,nameinlink]{cleveref}
\Crefname{assumption}{Assumption}{Assumptions}
\Crefname{postulate}{Postulate}{Postulates}
\crefname{figure}{Figure}{Figures}
\crefname{principle}{Principle}{Principles}

\newcommand{\indep}{\perp\mkern-9.5mu\perp}

\makeatletter

\newcommand{\Rom}[1]{\expandafter\@slowromancap\romannumeral #1@}

\makeatother

\algnewcommand{\LineComment}[1]{\Statex \(\triangleright\) #1}

\newcommand{\q}{\mathbf{q}}

\newcommand{\bX}{X}
\newcommand{\mX}{\boldsymbol{X}}

\newcommand{\bZ}{Z}
\newcommand{\bM}{M}
\newcommand{\bN}{N}

\newcommand{\bL}{L}

\newcommand{\sx}{\Sigma_{\bX\bX}}
\newcommand{\sn}{\Sigma_{\bN\bN}}

\DeclareMathSymbol{\ast}{\mathbin}{symbols}{"03}

\DeclareMathOperator{\pa}{{pa}}      % for parents
\DeclareMathOperator{\ch}{{ch}}      % for parents
\DeclareMathOperator{\an}{{an}}   % for ancestors
\DeclareMathOperator{\de}{{de}}   % for desendents 
\makeatletter
\newcommand*\bigcdot{\mathpalette\bigcdot@{.5}}
\newcommand*\bigcdot@[2]{\mathbin{\vcenter{\hbox{\scalebox{#2}{$\m@th#1\bullet$}}}}}
\makeatother

\usepackage[textwidth=1.9cm]{todonotes}

\usepackage{tcolorbox}

\usepackage{tikz}
\usetikzlibrary{arrows}
\usetikzlibrary{calc,decorations.pathmorphing,patterns,positioning}
\usetikzlibrary{arrows.meta,bending,shapes,shapes.symbols}
\tikzstyle{dot_node}=[draw=black,fill=black,shape=circle]

\tikzset{
    vertex/.style={circle, draw, inner sep=2pt, minimum size=20pt}, %, draw=darkgreyblue,fill=greyblue, minimum size = 66pt,  line width=2pt, outer sep=3pt, inner sep=6pt,
    label/.style={above, pos=0.4, fill=white,  inner sep=1pt},
    edge/.style={-latex} %line width=3pt,darkgreyblue,> = {stealth'[length=6mm,width=5mm]},
    }
\newcommand{\latentpath}
{\hspace{0.3ex}\tikz[baseline=-0.5ex] \draw[->, >=To, decorate, decoration={snake, segment length=2mm, amplitude=0.3mm}, line width=0.5pt, dashed,  dash pattern=on 2pt off 1pt] (0,0) -- (0.6,0);\hspace{0.3ex}}
\newcommand{\patharrow}{\hspace{0.3ex}\tikz[baseline=-0.5ex] \draw[->, >=To, decorate, decoration={snake, segment length=2mm, amplitude=0.3mm}, line width=0.5pt] (0,0) -- (0.6,0);\hspace{0.3ex}}
\newcommand{\reversedlatentpath}{\hspace{0.3ex}\tikz[baseline=-0.5ex] \draw[<-, >=To, decorate, decoration={snake, segment length=2mm, amplitude=0.3mm}, line width=0.5pt, dashed,  dash pattern=on 2pt off 1pt] (0,0) -- (0.6,0);\hspace{0.3ex}}

\tikzset{
  ziczacdash/.style={
    decorate,
    decoration={
      zigzag,
      amplitude=4pt,
      segment length=14pt
    }
  }
}

\definecolor{icefire_light_blue}{HTML}{6BB2CD}
\definecolor{icefire_blue}{HTML}{377ad0}
\definecolor{icefire_purple}{HTML}{7571E8}
\definecolor{icefire_pink}{HTML}{9d3d55}
\definecolor{icefire_orange}{HTML}{EF8748}

\newcommand{\ziczac}{\tikz[baseline=-0.7ex]\draw[ziczacdash, <->, dashed,  dash pattern=on 2pt off 1pt, line width=0.5pt] (0,0) -- (1.78,0);}

\title{}

\date{March 2025}

\begin{document}



\twocolumn[

\aistatstitle{Root Cause Analysis of Outliers in Unknown Cyclic Graphs}

\aistatsauthor{ Daniela Schkoda \And Dominik Janzing}

\aistatsaddress{Amazon Research T\"ubingen, Germany \\ Technical University of Munich, Germany \And  Amazon Research T\"ubingen, Germany } ]

\begin{abstract}
  We study the propagation of outliers in cyclic causal graphs with linear structural equations, tracing them back to one or several ``root cause'' nodes. We show that it is possible to identify a short list of potential root causes provided that the perturbation is sufficiently strong and propagates according to the same structural equations as in the normal mode. This shortlist consists of the true root causes together with those of its parents lying on a cycle with the root cause.
   Notably, our method does not require prior knowledge of the causal graph and yields encouraging results on simulated data and real data from biology and cloud computing.
\end{abstract}

\section{INTRODUCTION}
Root cause diagnosis is a fundamental challenge across a wide range of domains. Whether tracing the origin of a failure in a cloud-based microservice architecture, identifying a perturbed gene in a gene-regulatory network, locating the source of a shock in a financial system, or detecting faults in industrial manufacturing processes, the need to accurately pinpoint the source of disruptions is critical. For instance, in microservice-based cloud applications, when a failure  occurs in one service, its effects can propagate across others, creating a cascade of symptoms and thousands of alerts. Root cause analysis (RCA) is thus critical for maintaining reliability and availability. 

RCA closely aligns with the task of intervention targets estimation (ITE). 
While the framing differs slightly, in both cases, interventions on a few nodes in a causal system explain shifts in the overall distribution, and the goal is to identify these nodes. 

Existing methods for both tasks largely fall into two categories. The first one includes graph-based methods, which require (or first learn) a ground truth causal graph \citep{RCA_graph_kim2013root, RCA_graph_assaad23a,Tao_2024, xie2024cloudatlasefficientfault, RCA_graph_tonon2025radice}. However, in practice, obtaining a complete ground-truth model for large systems is typically infeasible, and causal discovery methods to learn them from data often rely on restrictive and unverifiable assumptions. 
To address this, a second line of work performs RCA without relying on a known ground truth: For instance, \cite{ikram2022root} only partially reconstructs the underlying graph to effectively identify root causes, while \cite{shan2019epsilon} leverages time series data for root cause detection. \cite{Varici2021, Varici2022, Wang2018, Ghoshal2019, Malik2024}  learn the intervention targets by inspecting the difference of the precision matrices of the observational and interventional data, and \cite{Yang2024} by applying contrastive learning to recover exogenous noises that shift between both data sets. \cite{Rothenhaeulser2015} requires data from at least three distinct environments and recovers the targets by joint matrix diagonalization of differences between the covariance matrices.

A shared limitation across all these works is the reliance on multiple anomalous samples. 
In contrast, \cite{RCA_graph_okati2024rca} introduce two RCA approaches that operate with only a single anomalous sample: \textit{Smooth Traversal}, which again requires ground-truth knowledge of the causal DAG (but not of the conditional distributions or even SCMs), and \textit{Score Ordering}, which avoids such knowledge but assumes that the unknown ground truth follows a polytree structure.
The work most closely related to ours is \cite{Li2024}, which assumes a linear structural equation model (SEM) and demonstrates that a permutation search combined with Cholesky decomposition of the covariance matrix can uncover the root cause, even in the case of a single interventional sample. 
However, their approach still requires an \textit{acyclic} ground-truth structure, which is again often violated in practice: gene regulatory networks routinely contain cyclic feedback loops, and microservice calls can form cycles, e.g., through repeated calls after failures \citep{microservice_cycles, microservice_cycles_2}.

In this work, we address exactly this gap. We consider the task of root cause analysis in settings described by SEMs that may include cycles without requiring ground truth. Our approach is heavily inspired by \cite{Li2024} as it also uses the covariance matrix to derive a simple linear transformation that reveals root causes when applied to the vector of outlier observations. However, it is conceptually more straightforward and avoids the expensive combinatorial search needed there, making it more computationally efficient; see the appendix for 
a more detailed comparison of the ideas.

\section{NOTATION AND MODEL} 
Our method assumes access to data under the normal regime, along with one anomalous sample for which we seek to identify the root causes of the abnormality. We model the $p$-variate distribution under the normal regime, denoted by $P^{\bX}$, using a linear cyclic SEM \citep{Richardson1996, Spirtes2001, bongers2018theoretical}. 
Specifically, each feature $X_i$ of $X=(X_1, \dots, X_p)$ satisfies \begin{equation} X_i = \sum_{j=1}^p a_{ij} X_j + N_i, \end{equation} or, in matrix-vector-form,
\begin{equation}\label{eq:cyclic_SEM}
    \bX = A \bX + \bN, 
\end{equation}
where $A$ is a coefficient matrix with zero diagonal entries, and $\bN$ is a noise vector with mean zero and uncorrelated components\footnote{Unlike many works in causal discovery, we do not require the noise terms to be mutually independent; uncorrelatedness suffices.}. We further assume that $I-A$ is invertible. The matrix $A$ defines the system's causal structure, with a non-zero entry $a_{ij}$ indicating a direct causal effect from variable $X_j$ to variable $X_i$. This structure is naturally represented by a directed graph $\mathcal{G}$, where an edge $X_j \to X_i$ exists if and only if $a_{ij} \neq 0$.
We assume that the same linear relations still hold in the anomalous mode up to a sparse additive perturbation term $\Delta$:
\begin{equation}\label{eq:cyclic_SEM_anormal_mode}
    \tilde{X} = A \tilde{X} + N' + \Delta, 
\end{equation} where $N'$ is an i.i.d. copy of $\bN$ and $\Delta$ has non-zero entries for the root cause nodes, which we denote by $\mathcal{R}$. We may also absorb the perturbation into the noise by defining $\tilde{N}=N' + \Delta.$ 

Note that modelling the perturbation as a change of the noise does not mean that 
we assume a perturbation affecting the exogenous noise in any ``material'' sense.
For example, assume that the root cause manifests as a change in the causal mechanisms, 
 $\tilde{X}_r = f(\text{pa}(r), \tilde{N}_r)$, whereas in the normal mode, $X_r = \sum_{i \in \text{pa}(r)}a_{ri}X_i + N_r$. This modification can be  ``absorbed'' into a perturbation term $\delta$ such that it matches our framework:
$\tilde{X}_r = \sum_{i \in \text{pa}(r)}a_{ri}X_i + N'_r + \delta$,
where $\delta = f(\text{pa}(r), \tilde{N}_r) - \sum_{i \in \text{pa}(r)}a_{ri}X_i - N_r$; compare also \cite{backtracking}.

Within this setup, the task of RCA lies in identifying which entries of the noise vector in 
$\tilde{N}$ are extreme when compared to the original noise vector
$\bN$. To make the connection between $\bX$ and $\bN$ more explicit, we can  rearrange \eqref{eq:cyclic_SEM} as
% \[
% (I - A) \bX = \bN,
% \]
% and subsequently solve for $\bX$ to obtain 
\[
\bX = (I-A)^{-1} \bN. 
\]
The entries of $(I-A)^{-1}$ capture \textit{total causal effects} and are related to paths in the graph $\mathcal{G}$. A \textit{path} from $i_1$ to $i_k$, denoted by  $i_1 \patharrow i_k$, is a sequence of edges $(i_1 \to i_2, i_2 \to i_3, \dots, i_{k-1} \to i_k)$ such that neither nodes nor edges are repeated, except possibly for a return to the starting node, $i_1 = i_k$, forming a cycle. We allow paths of length $0$. The following standard result will become important later: 
\begin{lemma}[Path matrix, \cite{lauritzen1996graphical}]\label{lem:path_matrix}  The total causal effect from $j$ on $i$ multiplies direct causal effects along paths from $j$ to $i$:
\begin{multline*}
    ((I-A)^{-1})_{ij}  = \frac{1}{\det(I-A)}  \sum_{\pi: j \rightsquigarrow i}
        \prod_{k\to h\in \pi}  a_{hk} \\ \cdot \left(1+
        \sum_{\substack{(C_1, \dots, C_q)\in \mathcal{C}_\pi,\\q\geq 1.}} (-1)^{q} \prod_{m =1}^q  \prod_{k\to h\in C_m}  a_{hk}\right), 
\end{multline*} where $\mathcal{C_\pi} = \{(C_1, \dots, C_q):$ collection of disjoint cycles in $G$ not using any node in $\pi; q\in \mathbb{N}_0.\}$. 
\end{lemma}
While the proof is provided in the supplementary material, especially for acyclic graphs, the intuition is quite straightforward:  the total causal effect accounts for all paths - direct and indirect - through which one variable can influence another. Each path represents a potential causal route, and its contribution is the product of the direct effects along it. This motivates to call $(I- A)^{-1}$ the  \textit{path matrix}.
% Note that the identity $((I- A))^{-1}_{ij}=1$ for $i = j$ follows from allowing paths of length 0. 
As our primary interest lies in the sparsity pattern of  the matrix, we omit the explicit determinant formula here and refer the reader to the proof of Lemma 4.5 in \cite{Amendola2020}.

Many of the results proven in this paper assume \textit{genericity} of the parameters $A$ and $\Sigma_{NN}$, meaning that the results might fail on a Lebesgue measure zero set of matrices $A$ (consistent with a fixed graph $\mathcal{G}$) and choices of $\Sigma_{NN}$. This assumption rules out degenerate cases, such as interventions on two nodes that cancel each other. For example, in the chain
\[X_1 
\to X_2 \to X_3
\] with all edge weights being one and two root causes $X_1, X_2$ with $\Delta_1=1, \Delta_2=-1$, we obtain $\tilde{X}_2 = N_1 + N_2 + \Delta_1 + \Delta_2 = N_1+N_2$, which is not anomalous although $X_2$ is a root cause. Our method cannot identify $X_2$ as the root cause in this pathological case. A precise characterization of these exceptional situations, as well as most proofs, are deferred to the appendix. Finally, we denote the children, parents, descendants, and ancestors of a node $X_i$ by $\ch(X_i), \pa(X_i), \de(X_i)$, and $\an(X_i)$, respectively.

Having established the model, we discuss how Root Cause Analysis can be solved in this framework.
\section{ROOT CAUSE ANALYSIS VIA INVARIANT STRUCTURAL EQUATIONS}
\subsection{Causally Sufficient Cyclic Model}
We first consider the case of a single root cause, that is $\Delta = (0, \dots, 0, \delta, 0, \dots, 0)$ with $\delta \in \mathbb{R}$ representing a significant perturbation at the root cause $r$. % While this assumption is not necessary, we adopt it initially to simplify the presentation. 
A key ingredient of our method is the precision matrix, which can be conveniently expressed in terms of the path matrix. Specifically,
\[
\sx = (I- A)^{-1} \sn (I-A)^{-T}, 
\]
where $\sn$ denotes the covariance matrix of 
the noise, which is diagonal by assumption:
$\sn = {\rm diag}(\sigma_1^2,\dots, \sigma_p^2)$.
Inverting this expression yields the precision matrix:
\[
\Theta_{XX} := \sx^{-1} = (I-A)^{T} \Theta_{NN} (I-A),
\] where $\Theta_{NN}$ denotes the precision matrix of $N$.
To see how this relationship supports root cause analysis, consider 
%\begin{eqnarray}\label{eq:definition_xi}
$$\xi := \Theta_{XX} \tilde{X} = \Theta_{XX}(I- A)^{-1}( \Delta + N').$$
%\end{eqnarray}
Inspecting the first summand, we observe
\begin{align}
\Theta_{XX}(I- A)^{-1} \Delta &= (I-A)^{T} \Theta_{NN} (I-A) (I- A)^{-1} \Delta  \nonumber
\\ &=
  (I - A)^T \Theta_{NN} \Delta \nonumber \\ 
  &= \frac{1}{\sigma_r^2} \delta (I - A)^T e_r  =: \xi', \label{eq:xi}
\end{align} where $e_r$ denotes the $r$th canonical basis vector.
Since the $r$th column of $(I - A)^T$ is non-zero precisely at the positions of $r$ and its parents, also $\xi'$ is non-zero only at these positions. 
Turning to the second summand, 
$
\Theta_{XX}(I- A)^{-1}N',  
$ note that it has the same distribution as 
$\Theta_{XX} \bX$ since $N' \stackrel{d}{=} \bN$. Therefore, by evaluating which $\xi_i$ has the same distribution as $\Xi_i:=(\Theta_{XX} \bX)_i$, we can separate the root cause and its parent from the remaining nodes. Especially, if $\delta$ has high absolute value, we can expect the entries of the root cause and its parents to differ significantly from all other entries.
Overall, we derived:
% \begin{lemma}\label{lem:bi}
% If two nodes $i,j$ in a directed (potentially cyclic) graph
% satisfy $\{i\} \cup Ch(i) = \{j\} \cup Ch(j)$, then  there is
% a bidirected edge $i \leftrightarrow i$. 
% \end{lemma}
% \begin{proof}
% Assume $\{i\} \cup Ch(i) = \{j\} \cup Ch(j)$ but $i\neq j$.
% Then $i \in Ch(j)$ and $j\in Ch(i)$. 
% \end{proof}

% Together with Lemma \ref{lem:bi} we have thus shown that 
% we can localize the root cause up to a set of candidates, namely the set of nodes connected with the root cause with directed edges. 
\tcbset{colback=black!5, colframe=black, boxrule=0.2mm, arc=2mm}
\begin{tcolorbox}
Applying the precision matrix to the outlier vector leaves only the root causes and their parents as extreme, thereby revealing them.\end{tcolorbox}\vspace*{0.2cm}

\begin{theorem}[Shortlist of root causes]\label{thm:xi} For generic $A$ and $\Sigma_{NN}$, all entries in $\xi_i$ follow their usual distribution, except the root causes and potentially their parents:
\[\mathcal{R}  \subseteq \{\xi_i \centernot{\stackrel{d}{=}}\Xi_i\} \subseteq \mathcal{R} \cup \pa(\mathcal{R}).\]
% \[
% \xi_i \centernot{\stackrel{d}{=}}\Xi_i \iff i \in \mathcal{R} \cup \pa(\mathcal{R}).
% \]
\end{theorem}
To exemplify this result, we  consider a simple two-dimensional toy example.
\begin{example}
    Assume the data-generating process is described by the causal DAG $X_1 \to X_2$, and both $X_1$ and $X_2$ are standard normally distributed. The structural equations are \begin{eqnarray*} X_1 &=& N_1, \\ X_2 &=& \rho X_1 + N_2, \end{eqnarray*} where $\rho$ denotes the Pearson correlation between $X_1$ and $X_2$. Further, the covariance matrix and its inverse  read
\[
\sx = \left(\begin{array}{cc} 1 & \rho \\ \rho & 1
\end{array} \right), \quad
\Theta_{XX} = \frac{1}{1- \rho^2} \left(\begin{array}{cc} 1 & -\rho \\ -\rho & 1
\end{array} \right).
\]
If $X_1$ is the root cause, i.e., $\Delta = (\delta, 0)^T$, then
$\tilde{X} \overset{d}{=} X +  (\delta, \delta \cdot \rho)^T$. Multiplying $\tilde{X}$ by $\Theta_{XX}$ yields
\[
\Theta_{XX} \tilde{X} \overset{d}{=} \Theta_{XX} X + \Theta_{XX} \delta \left(\begin{array}{c} 1 \\ \rho \end{array} \right) = \Theta_{XX} X+ \left(\begin{array}{c} \delta \\ 0 \end{array} \right),\]
where only the first entry is extreme for $\delta\to \infty$.
On the other hand, if $X_2$ is the root cause, we have $\Delta=(0,\delta)^T$ and $\tilde{X}\overset{d}{=} X + (0,\delta)^T$ and obtain 
\begin{align*}
    \Theta_{XX}\tilde{X} &\overset{d}{=} \Theta_{XX} 
 X+ \Theta_{XX} \left(\begin{array}{c} 0\\ \delta \end{array} \right) \\ &= \Theta_{XX} X + \delta \cdot \frac{1}{1- \rho^2} \left(\begin{array}{c} -\rho \\ 1 \end{array} \right),
\end{align*}
where both entries of the vector get extreme for high $\delta$, hence we can tell the difference between the two scenarios.
%Since we do not know the causal direction between $X_1$ and $X_2$, we would ideally like to distinguish between four scenarios: (i) $X_1\to X_2$ with $x_1$ root cause, (ii) $X_1\to X_2$ with $x_2$ root cause, (iii) $X_2\to X_1$ with $x_1$ root cause, (iv) $X_2\to X_1$ with $x_2$ root cause. We only fail to distinguish between (ii) and (iii).  
\end{example}

This result gives a first short list of candidate root causes. To restrict the list further, we can make use of the fact that only the descendants of the root causes  are anomalous. 
\begin{theorem}[Even shorter list of root causes]\label{thm:outliers_anomalous} Assuming generic parameters $A, \Sigma_{NN}$, an abnormal distribution of both, $\xi$ and $\tilde{X}$, occurs only for the root cause and for its parents that are simultaneously descendants: \begin{multline*}
\mathcal{R}  \subseteq \{\xi_i \centernot{\stackrel{d}{=}}\Xi_i \text{ and } \tilde{X}_i \centernot{\stackrel{d}{=}} \bX_i \} 
\subseteq \mathcal{R} \cup (\pa(\mathcal{R})\cap \de(\mathcal{R})).\end{multline*}
% \[
% \xi_i \centernot{\stackrel{d}{=}}\Xi_i \text{ and } \tilde{X}_i \centernot{\stackrel{d}{=}} \bX_i  \iff i \in \mathcal{R} \cup (\pa(\mathcal{R})\cap \de(\mathcal{R})).
% \]
\end{theorem}
Note that Theorem \ref{thm:outliers_anomalous}
allows for unique identification of the root cause (if there is only one) in acyclic graphs as the parents of the root cause cannot be descendants at the same time.

Another challenge %, besides the presence of multiple root causes, 
is dealing with latent variables -- factors that influence the system but are not measured.

 \subsection{Latent Cyclic Model}
Latent variables can be incorporated into the model by extending the variable set. We consider a random vector $\bZ = (\bX, \bL)$, where $\bX = (X_1, \dots, X_p)$ collects the observed variables and $\bL = (L_1, \dots, L_q)$ the latent variables. Overall, $\bZ$ captures all variables relevant to the underlying causal system and satisfies a linear SEM
\begin{equation}\label{eq:latent_model}
    \bZ = A \bZ + \bN
    \quad \hbox{ and } \quad 
    \tilde{Z} = A\tilde{Z} + N' + \Delta,
\end{equation}
which can be partitioned as \[
\begin{pmatrix} \bX \\ \bL \end{pmatrix}
= \begin{pmatrix}
    A_{\bX\bX} & A_{\bX\bL} \\
    A_{\bL\bX} & A_{\bL\bL}
\end{pmatrix} \begin{pmatrix} \bX \\ \bL \end{pmatrix} + \begin{pmatrix} N_{\bX} \\ N_{\bL} \end{pmatrix},
\] similarly for the interventional equation.
We aim to express the system in terms of $\bX$ only, by incorporating the latent components into the noise term, at the cost of losing noise uncorrelatedness. To this end, denote 
 $$S = A_{\bX\bX} + A_{\bX\bL}(I -  A_{\bL\bL})^{-1}A_{\bL\bX}, 
$$ and let $D$ the diagonal matrix of the same shape with entries $d_{ii}=\frac{1}{1-s_{ii}}$. Further define 
\begin{align*}
    \bar{A}_{ij} &= (D S)_{ij} \text{ for } i\neq j, \hbox{ and } 0 \text{ otherwise,} \\
    \bar{N} &= D (N_{\bX} + A_{\bX\bL}(I-A_{\bL\bL})^{-1}N_{\bL}),\\
    \bar{\Delta} &= D (\Delta_{\bX} + A_{\bX\bL}(I-A_{\bL\bL})^{-1}\Delta_{\bL}).
    % \bar{N} &= D (N_{\bX} + (I-S)((I-A)^{-1})_{\bX\bL}N_{\bL}),\\
    % \bar{\Delta} &= D (\Delta_{\bX} + (I-S)((I-A)^{-1})_{\bX\bL}\Delta_{\bL}).
\end{align*}
\begin{lemma}[Remove nodes] If $P^{\bZ}$ satisifies the SEM \eqref{eq:latent_model}, then $P^{\bX}$ and $P^{\tilde{X}}$ fulfill 
\[
    \bX = \bar{A}\bX+\bar{\bN} \quad \hbox{ and } \quad 
    \tilde{X} = \bar{A}\tilde{X}+\bar{N'}+\bar{\Delta}.
\]
\end{lemma}
\begin{proof} We show the observational SEM; the interventional one works analogously.
As a first step, we prove that \begin{equation}\label{eq:intermediate_eq_system_latent_model}
    \bX = S\bX+\bM,
\end{equation} where $\bM=\bN_{\bX} + A_{\bX\bL}(I-A_{\bL\bL})^{-1}{\bN_{\bL}}$. %via showing that $\bX = (I-A')^{-1}\bar{N}$. 
Note that
\[
I - S = \left( (I - A_{\bX\bX}) - A_{\bX\bL} (I - A_{\bL\bL})^{-1} A_{\bL\bX} \right)
\]
is precisely the Schur complement \citep{Schur1917} of the block $I - A_{\bL\bL}$ in $I - A$, yielding that
\[
((I - A)^{-1})_{\bX \bX} = (I - S)^{-1},
\] and, 
\[
((I-A)^{-1})_{\bX\bL} = (I-S)^{-1}A_{\bX\bL}(I-A_{\bL\bL})^{-1}.
\] Combining both yields
\begin{align*}
    (I - S)^{-1} \bM = &((I - A)^{-1})_{\bX\bX} N_{\bX} \\ &+ (I - S)^{-1} A_{\bX\bL}(I-A_{\bL\bL})^{-1} N_{\bL} \\
    = &((I - A)^{-1})_{\bX\bX} N_{\bX} + ((I - A)^{-1})_{\bX \bL} N_{\bL} \\ = &\bX,
\end{align*}
which confirms the intermediate equation  \eqref{eq:intermediate_eq_system_latent_model}. 
While this representation accurately describes $\bX$, it does not yet constitute a valid SEM since $S$ may feature non-zero diagonal entries. To enforce zero diagonals, we apply a rescaling. Writing out the $i$-th component,
\[
X_i = a'_{ii} X_i + \sum_{j \neq i} a'_{ij} X_j + M_i,
\] then subtracting $a'_{ii} X_i$, and rescaling both sides gives the SEM stated in the lemma:
\[
X_i = \frac{1}{1 - a'_{ii}} \left( \sum_{j \neq i} a'_{ij} X_j + M_i \right).
\]
\end{proof}
\begin{example}
    \begin{figure}
    \centering
    \begin{subfigure}{0.3\textwidth}
    \begin{tikzpicture}[thick]
      \node[vertex] (1) at (0,2) {$X_1$};
      \node[vertex] (3) at (1,0) {$X_3$};
      \node[vertex, fill=icefire_orange!30, draw=icefire_orange!20!Orange, text=icefire_orange!20!Orange] (L1) at (1.4,2.6) {$L_1$};
      \node[vertex, fill=gray!20, anchor=center, xshift=-1cm, yshift=1cm] (e1) at (1) {$N_1$};
      \node[vertex, fill=gray!20, anchor=center, xshift=1cm, yshift=1cm] (e3) at (3) {$N_3$};
      \node[vertex, draw=icefire_blue, text=icefire_blue, fill=icefire_blue!20, anchor=center, xshift=-.6cm, yshift=-1.3cm] (L2) at (1) {$L_2$};
      \node[vertex, anchor=center, xshift=-.6cm, yshift=-1.3cm] (2) at (L2) {$X_2$};
      \node[vertex, fill=gray!20, anchor=center, xshift=-1cm, yshift=1cm] (e2) at (2) {$N_2$};
      \node[vertex, fill=gray!20, anchor=center, xshift=-1cm, yshift=1cm] (e5) at (L2) {$N_5$};
      \node[vertex, fill=gray!20, anchor=center, xshift=-1cm, yshift=1cm] (e4) at (L1) {$N_4$};
      \draw[edge, icefire_orange!20!Orange] (L1) -- node[above, fill=white, inner sep=1pt] {\footnotesize $a_{14}$} (1);
      \draw[edge, icefire_orange!20!Orange] (L1) -- node[midway, fill=white, inner sep=1pt] {\footnotesize $a_{34}$}  (3);
      \draw[edge, icefire_blue] (L2) -- node[near start, fill=white, inner sep=1pt] {\footnotesize $a_{25}$} (2);
      \draw[edge, icefire_blue] (1) -- node[near start, fill=white, inner sep=1pt] {\footnotesize $a_{51}$} (L2);
      \draw[edge] (2) -- node[midway, fill=white, inner sep=1pt] {\footnotesize $a_{32}$} (3); 
      \draw[edge] (1) -- node[midway, fill=white, inner sep=1pt] {\footnotesize $a_{31}$} (3); 
      \draw[edge] (e1) -- node[near start, fill=white, inner sep=1pt] {\footnotesize $1$} (1);
      \draw[edge] (e2) -- node[near start, fill=white, inner sep=1pt] {\footnotesize $1$} (2);
      \draw[edge] (e3) -- node[near start, fill=white, inner sep=1pt] {\footnotesize $1$} (3);
      \draw[edge] (e5) -- node[near start, fill=white, inner sep=1pt] {\footnotesize $1$} (L2);
      \draw[edge] (e4) -- node[near start, fill=white, inner sep=1pt] {\footnotesize $1$} (L1);
    \end{tikzpicture}  \caption{ }\label{fig:graph_with_latents}  \end{subfigure}
    \hspace{1cm}
\begin{subfigure}{0.4\textwidth}
    \begin{tikzpicture}[thick]
    \node[vertex] (1) at (0,2) {$X_1$};
      \node[vertex] (3) at (1,0) {$X_3$};
      % \node[vertex, fill=icefire_orange!20!Orange!20, draw=icefire_orange!20!Orange, text=icefire_orange!20!Orange] (L1) at (1.4,2.6) {$L_1$};
      \node[ellipse, fill=icefire_orange!30, draw=icefire_orange!20!Orange, text=icefire_orange!20!Orange,  inner sep=2pt, anchor=center, xshift=-1.2cm, yshift=1cm] (e1) at (1) {$N_1+a_{14}L_1$};
      \node[ellipse, fill=icefire_orange!30, draw=icefire_orange!20!Orange, text=icefire_orange!20!Orange,  inner sep=2pt, anchor=center, xshift=1cm, yshift=1.2cm] (e3) at (3) {$N_3+a_{34}L_1$};
      \node[vertex, anchor=center, xshift=-1.2cm, yshift=-2.6cm] (2) at (1) {$X_2$};
      % \node[vertex, fill=icefire_blue!20, text=icefire_blue, draw=icefire_blue] (L2) at (-1.2,1.2) {$L_2$};
      \node[ellipse, inner sep=2pt, fill=icefire_blue!20, text=icefire_blue, draw=icefire_blue, anchor=center, xshift=-1cm, yshift=1.3cm] (e2) at (2) {$N_2+a_{25}N_5$};
      % \draw[edge, icefire_orange!20!Orange] (L1) -- node[above, fill=white, inner sep=1pt] {\footnotesize $a_{14}$} (e1);
      % \draw[edge, icefire_orange!20!Orange] (L1) -- node[midway, fill=white, inner sep=1pt] {\footnotesize $a_{34}$}  (e3);
      \draw[edge, icefire_blue] (1) -- node[near start, fill=white, inner sep=1pt] {\footnotesize $a_{51}a_{25}$} (2);
      \draw[edge] (2) -- node[midway, fill=white, inner sep=1pt] {\footnotesize $a_{32}$} (3); 
      \draw[edge] (1) -- node[midway, fill=white, inner sep=1pt] {\footnotesize $a_{31}$} (3); 
      %\draw[edge, icefire_blue] (L2) -- node[near start, fill=white, inner sep=1pt] {\footnotesize $a_{25}$} (e2); 
      \draw[edge] (e1) -- node[near start, fill=white, inner sep=1pt] {\footnotesize $1$} (1);
      \draw[edge] (e2) -- node[near start, fill=white, inner sep=1pt] {\footnotesize $1$} (2);
      \draw[edge] (e3) -- node[near start, fill=white, inner sep=1pt] {\footnotesize $1$} (3);
    \end{tikzpicture}
    \caption{}\label{fig:projeting_SEMb}
    \end{subfigure}
    \caption{Projecting a linear SEM to an SEM for only the observed nodes.} 
    \label{fig:projeting_SEM}
\end{figure}
 For the SEM depicted in Figure \ref{fig:graph_with_latents}, that is,
\[\begin{pmatrix}
    X_1 \\ X_2 \\ X_3 \\ L_1 \\ L_2
\end{pmatrix} =
\begin{pmatrix}
    0 & 0 & 0 & a_{14} & 0 \\
    0 & 0 & 0 & 0 & a_{25} \\
    a_{31} & a_{32} & 0 & a_{34} & 0\\
    0 & 0 & 0 & 0 & 0\\
    a_{51} & 0 & 0 & 0 & 0
\end{pmatrix}
\begin{pmatrix}
    X_1 \\ X_2 \\ X_3 \\ L_1 \\ L_2
\end{pmatrix}
+
\begin{pmatrix}
     N_1 \\ N_2 \\ N_3 \\ N_4 \\ N_5
\end{pmatrix}
\]
the projected SEM, visualized in Figure \ref{fig:projeting_SEMb}, reads
\[\begin{pmatrix}
    X_1 \\ X_2 \\ X_3 
\end{pmatrix} =
\begin{pmatrix}
    0 & 0 & 0 \\
    a_{25}a_{51} & 0 & 0 \\
    a_{31} & a_{32} & 0
\end{pmatrix}
\begin{pmatrix}
    X_1 \\ X_2 \\ X_3 
\end{pmatrix}
+
\begin{pmatrix}
    N_1 + a_{14}L_1 \\
    N_2 + a_{25}L_2 \\
    N_3 + a_{34}L_3 \\
\end{pmatrix}.
\] Here, each latent variable is absorbed into the noise terms of its children, making the noises of nodes with a common latent parent correlated. Moreover, removing $L_2$ from the system introduces a direct link $X_1 \to X_2$, reflecting what was previously an indirect, mediated path.
\end{example}

In general graphs, $\bar{N}_i$ and $\bar{N}_j$ are correlated whenever they have a joint latent ancestor that connects to both through a path in which all inner nodes are latent, called \textit{latent path} henceforth, and represented by $Z_k \latentpath Z_l$. Similarly, $\bar{\Delta}_i \neq 0$ for all observed root causes and all $i$ which lie downstream of a latent root cause through a path consisting only of latent nodes. We refer to all these nodes $i$ as \textit{observable root causes}, aligning with the intuition that we cannot identify a root cause that was never observed. Instead, the observable root causes are the earliest observed nodes infected by the anomalies. 

From the projected SEM over $\bX$, the precision matrix can be computed as: \[
(\Sigma_{\bX \bX})^{-1} = (I-\bar{A})^T(\Sigma_{\bar{N} \bar{N}})^{-1}(I-\bar{A}).
\]
In contrast to the case without latent variables, the noise precision matrix $(\Sigma_{\bar{N} \bar{N}})^{-1}$ is not necessarily diagonal due to correlations introduced by marginalizing out latent nodes. Instead, its sparsity pattern can be described as follows. 
% Instead, $\left(K_{\bar{N}, \bar{N}}\right)_{ij}$ is not zero if and only if $\bar{N}_i \not\indep \bar{N}_j \mid \bar{N}_{[p]\setminus\{i,j\}}$ \todo{reference}, 
\begin{lemma}[Sparse noise precision matrix]
    $(\Theta_{\bar{N} \bar{N}})_{ij}=0$ if there exists no sequence of latent paths 
    $$X_i \reversedlatentpath L_{l_1} \latentpath X_{k_1} \reversedlatentpath L_{l_2} \latentpath X_{k_2} \reversedlatentpath \cdots \latentpath X_j,$$ with $l_{\iota} \in [q], k_{\iota} \in [p]$.
\end{lemma}
Following \cite{BOEGE2025104604}, we call such sequences of paths \textit{zig-zag-structures} and denote them by $X_i \ziczac X_j$.  For example, for $G = X_1 \leftarrow L_1 \to X_2 \leftarrow L_2 \to X_3$,
\[
\bar{\bN} = \begin{pmatrix}
    N_1 + a_{14}L_1 \\
    N_2 + a_{24}L_1 + a_{25}L_2 \\
    N_3 + a_{35}L_2
\end{pmatrix}.
\] While $\bar{N}_1$ and $\bar{N}_3$ are marginally independent, conditioning on $\bar{N}_2$ introduces dependence. 

As before, we define the score vector as
\begin{align*}
\xi &=\Theta_{\bX \bX}\tilde{X}
\\ &= (I - \bar{A})^T \Theta_{\bar{N} \bar{N}} \bar{\Delta} + (\Sigma_{\bX \bX})^{-1}(I- \bar{A})^{-1}\bar{N'}
\end{align*}
Since the first term dominates for large perturbations, we are again interested in its support.
%If $r \in O$, $\tilde{\Delta} = \Delta_O$, otherwise $\tilde{\Delta} = (I-\tilde{A})((I-A)^{-1})_{O,r}$, which has non-zero entries at all positions $i$ such that $r \latentpath i$. Let's call these nodes observable root causes. Then, 
From the Lemma, $\Theta_{\bar{N} \bar{N}}\bar{\Delta}$ is non-zero for all 
$X_j$ having a zig-zag-connection to an observable root cause. Furthermore, $\bar{A}$ encodes the child-parent relationships in the projected graph, that is, $a_{ij}\neq 0$ if and only if $X_j \latentpath X_i$. 
\begin{theorem}[Shortlist of root causes]\label{thm:sparisty_xi_observed} If $A$ and $\Sigma_{NN}$ are generic, then
  $\xi_i \centernot{\stackrel{d}{=}} (\Theta_{XX} \bX)_i$ if and only if $X_i$ is connected to an observable root cause $X_k$ through the structure
\begin{equation*}
    X_i\latentpath X_{j} \ziczac X_k.
\end{equation*} 
\end{theorem}
As before, this shortlist can be restricted further by removing non-anomalous nodes.
% A specific case here are microservice graphs with multiple measurements per service. These vector-valued nodes can be transformed into a scalar-valued where there are no edges between measurements of same service, but the noises of all measurements of the same service are correlated. If there is no further confounding, no longer zic-zac structures with more than one top node is possible. Moreover, in this setting, all latents are source nodes, such that there are no latent directed paths from one observed node to another. Hence, the structure reduces to 
% \[j\to j_1 \leftarrow l_1 \to r.\]
% and $j_1$ and $r$ belong to the same service. So we simply get (on a service level) the root cause and its parents.
 Besides the obvious application of RCA in systems including latent variables, this theory also enables tracing the routes through which anomalies propagate.
\begin{example}[Find root cause propagation route]
Assume that $G$ is causally sufficient and that there is a unique root cause $X_r$. In a first iteration, we apply our procedure to identify it. Then, a natural next step is identifying which nodes were subsequently affected, i.e., infer
the children of $X_r$. To this end, we remove $X_r$ from the data, making it a latent node. 
On this reduced data set, we apply our method again, yielding a new score vector $\breve\xi \in \mathbb{R}^{p-1}$. From Theorem  \ref{thm:sparisty_xi_observed},  it is non-zero precisely at the positions $i$ where there is a route
\begin{equation*}
    X_i{\color{icefire_purple}\latentpath} X_{j} {\color{icefire_pink}\ziczac} X_{k}  {\color{icefire_orange}\reversedlatentpath} X_r.
\end{equation*} 
Since $X_r$ is the sole latent node in this zigzag-structure, the latter collapses to either 
$$X_i{\color{icefire_purple}=} X_{j} {\color{icefire_pink}\leftarrow X_r \to \cdots \leftarrow X_r \to} X_{k} {\color{icefire_orange}\leftarrow} X_r,$$
$$X_i  {\color{icefire_purple}\to} X_{j} {\color{icefire_pink}\leftarrow X_r  \to \cdots \leftarrow X_r \to} X_{k}  {\color{icefire_orange}\leftarrow} X_r,$$
or to $$X_i {\color{icefire_purple}\to X_r \to }  X_{j} {\color{icefire_pink}\leftarrow X_r \to  \cdots \leftarrow X_r \to} X_{k}  {\color{icefire_orange}\leftarrow} X_r.$$
Thus, $X_i$ is a child of $X_r$, or a parent of either $X_r$ or of another child of $X_r$. So, again, we have a candidate set for the children. Reiterating this approach, we can find supersets of all routes through which the anomaly propagated, which, however, can grow exponentially.
\end{example}

\section{IMPLEMENTATION WITH FALSE DISCOVERY RATE CONTROL}\label{sec:implementation}
Turning to the practical implementation, we assume to have access to one anomalous sample $\tilde{X}$, alongside $m$ i.i.d samples $X^{(1)}, \dots, X^{(m)} \sim P^X$, collected into a data matrix $\mX \in \mathbb{R}^{m\times p}$. %The symbols $\hat{\Sigma}_{\mX \mX}$ denote its sample covariance matrix, and $\hat{\sigma}_{\mX_i}$ refers to the standard deviation of a single feature. Further, we denote $\hat{\xi} = \hat\Sigma_{\bX\bX}^{-1} \tilde{X}$ and $\hat{\boldsymbol{\Xi}} = \hat\Sigma_{\bX\bX}^{-1} \bX.$ 
We first focus on Theorems \ref{thm:xi} and \ref{thm:sparisty_xi_observed}, which identify candidate root causes as the nodes that violate the hypothesis
$$H_0: \xi_i \stackrel{d}{=} \Xi_i.$$
  For simplicity, we henceforth refer to these violating  nodes 
as \textit{positive}, while the others are termed \textit{negative}. A straightforward approach to assess $H_0$ is to use the empirical p-value, defined as the proportion of samples $|\Xi_i|$ exceeding $|\xi_i|$. %\[\ operatorname{pval}^{(2)}_{\mathrm{emp}}(\xi_i; \boldsymbol{\Xi}_i) =\frac{\#\{|{\boldsymbol{\Xi}}_i| \geq |{\xi}_i|\} + 1}{m+1}.\]
% \[\operatorname{pval}^{(2)}_{\mathrm{emp}}(\xi; \Xi)
% =\frac{\#\{|\hat{\boldsymbol{\Xi}}_i-\hat\mu_{\hat{\boldsymbol{\Xi}}_i}| \geq |\hat{\xi}_i-\hat\mu_{\hat{\boldsymbol{\Xi}}_i}|\} + 1}{n+1}.\]
While this guarantees valid Type I error control for incorrectly classifying a \textit{single} negative node as positive, it is challenging to use p-values with arbitrary dependencies to control the total number of false positives within a False Discovery Rate Control (FDRC) framework. Therefore, we instead turn to \textit{e-values}, where an e-value is a non-negative random variable $e$ satisfying
$\mathbb{E}(e) \leq 1$ under  $H_0$.
High values of $e$ indicate stronger evidence against the null hypothesis. Unlike p-values, e-values are particularly well-suited for multiple testing and aggregation, making them effective for FDRC; see \cite{e_values_vovk, e_values_Gruenwald} for further details. In our setting, \begin{equation}\label{eq:e-value}
e_i = \frac{\xi_i^2}{(\Theta_{XX})_{ii}}     
\end{equation}is a valid e-value since under $H_0$, $\xi_i$ is centered and  has  covariance  $(\Theta_{XX})_{ii}$, implying $\mathbb{E}(e)=1$. Note that this e-value coincides with a commonly used outlier score, namely the squared Z-Score defined via 
\[\text{Z-Score}(\xi_i; \boldsymbol{\Xi}_i)^2 = \left(\frac{\xi_i-\mu_{\boldsymbol{\Xi}_i}}{\sigma_{\boldsymbol{\Xi}_i}}\right)^2.\] 
The equivalence follows from the fact that $\boldsymbol{\Xi}$ is centered and has covariance $\Sigma_{\boldsymbol{\Xi} \boldsymbol{\Xi}} = \Theta_{\bX\bX}$.
We denote by $e_{[1]}, \dots, e_{[p]}$ the e-values $e_1, \dots, e_p$ arranged in decreasing order. Then, the e-value–based FDRC approach of \cite{Wang2022}  yields:
\begin{lemma}[Benjamini-Hochberg for e-values] Let $\alpha > 0$ be a desired FDR control level and define $k^*$ as the largest index $k$ such that $k e_{[k]} \geq 1/\alpha$.  Then, selecting the $k^*$ nodes corresponding to $e_{[1]}, \dots, e_{[k^*]}$ as candidate root causes yields
  \[ \text{FDR}=\mathbb{E}\left(\frac{\text{false positives}}{\text{total number of rejections}}\right)\leq \alpha.\]
\end{lemma} where the expectation is over the variables $X^{(1)}, \dots, X^{(m)} \sim P^X, \tilde{X}\sim P^{\tilde{X}}$.
%Unlike the Benjamini-Hochberg procedure for aggregating p-values, this FDRC is valid even when the $e_1, \dots, e_p$ are arbitrarily dependent.

% \begin{enumerate}
% \item Simply use the magnitude $|\xi_i|$ as a root cause score for each node. If the scale of $\delta$ is significantly larger than that of $A$ or $\Sigma_{\bN}$, the first summand in \eqref{eq:definition_xi} dominates the second, making this a reasonable and simple criterion.
% \item Compute the squared z-score of $\xi_i$ relative to the empirical distribution of $(\Sigma_{\bX}^{-1} \bX)_i$.
% \end{enumerate}
Finally, to incorporate the fact that the root cause is anomalous as stated in Theorem \ref{thm:outliers_anomalous}, we assign a root cause score of zero to all nodes  not qualifying as outliers, as measured by their squared Z-Scores falling below a threshold $\tau$. 
% Since precision matrix estimation is hard one alternative 
% to the practical implementation, note that estimating the precision matrix is hard. Therefore, we rather estimate the covariance $\hat{\Sigma}$ and estimate $\xi$ via soliving $\hat{\Sigma}\xi = \tilde{X}$. We still need , in \eqref{eq:e-value} we actually don't nned the $(\sx^{-1})_{ii}$ in \eqref{eq:e-value} for which we do estimate the precision matrix. however, please note that in practice we can expect that the main thing in here is larger deviations in $\xi_i$ whereas the scales $(\sx^{-1})_{ii}$ should not differ so much. SO its still prefarbale to have the 
The resulting procedure is summarized in Algorithm \ref{algorithm}.
\begin{algorithm}[!htbp]
\caption{Cyclic RCA}\label{algorithm}
 \begin{algorithmic}[1]
\Require Normal samples $\mX \in \mathbb{R}^{m \times p}$, anomalous sample $\tilde{X}\in \mathbb{R}^p$, threshold $\tau$ to characterize outliers.
\State $\hat{\Theta} \gets $ estimated precision matrix of $\mX$.
\State $\hat{\xi} \gets \hat{\Theta}\tilde{X}$.
\State $\hat{\boldsymbol{\Xi}} \gets \hat{\Theta}\mX$
\For{$i = 1$ to $p$}
    \State $s_i \gets \text{Z-Score}(\xi_i; \boldsymbol{\Xi}_i)^2$
    \If {Z-Score$(\tilde{X}_i; \bX_i)^2 <\tau$} 
        \State $s_i \gets 0.$
    \EndIf
\EndFor
\State \Return RCA scores $s=(s_1, \dots, s_p).$
 \end{algorithmic}
 \end{algorithm}
%Note that we use $s_i = 1-\operatorname{pval}^{(2)}_{\mathrm{emp}}(\xi_i; \boldsymbol{\Xi}_i)$, to ensure that for the p-value as well as for the e-value based approach, higher scores  indicate stronger root cause candidates.
The computational complexity of our algorithm breaks down to its bottleneck: estimation of the precision matrix. Using Graphical Lasso \citep{graphical_lasso} incurs a cost of $O(Tp^3)$, where $T$ is the iteration limit. A simpler option is to invert the sample covariance, with complexity $O(mp^2 + p^3)$. In both cases, space complexity is $O(mp + p^2)$. Further estimation methods are discussed in Appendix~\ref{sec:precision_and_compute_time}. 

Beyond the linear algebra formulation presented here, our method admits alternative (but slightly more involved) derivations, provided in appendix ~\ref{sec:alternative_derivations}.  One such view links our score
 $\text{Z-Score}(\xi_i; \boldsymbol{\Xi}_i)$ to linear regression. Furthermore, under certain assumptions, our 
 null hypothesis $\xi_i \stackrel{d}{=} \Xi_i$ can be rephrased as a conditional independence test, after introducing a binary indicator variable distinguishing anomalous from normal samples. This perspective clarifies the relation to conditional independence–based RCA \citep{Kocaoglu2019_CI, Jaber2020_CI, Mooij2020_CI, ikram2022root}, while also highlighting a key difference: owing to our parametric assumptions, we can perform the test with only a single anomalous sample.
\section{EXPERIMENTS}\label{sec:experiments}
More details on the experiments, the entailed statistical challenges, and further simulation studies can be found in the appendix.  All experiments were executed on a MacBook Pro (Apple M1, 36 GB RAM) or an internal cluster with 80 logical cores and 754 GB RAM, and the code is available at \url{https://github.com/DanielaSchkoda/CyclicRCA}.
\subsection{Simulation Studies}
%\begin{figure}[t]
\begin{figure}[!htb]
\centering
\includegraphics[width=.93\linewidth]{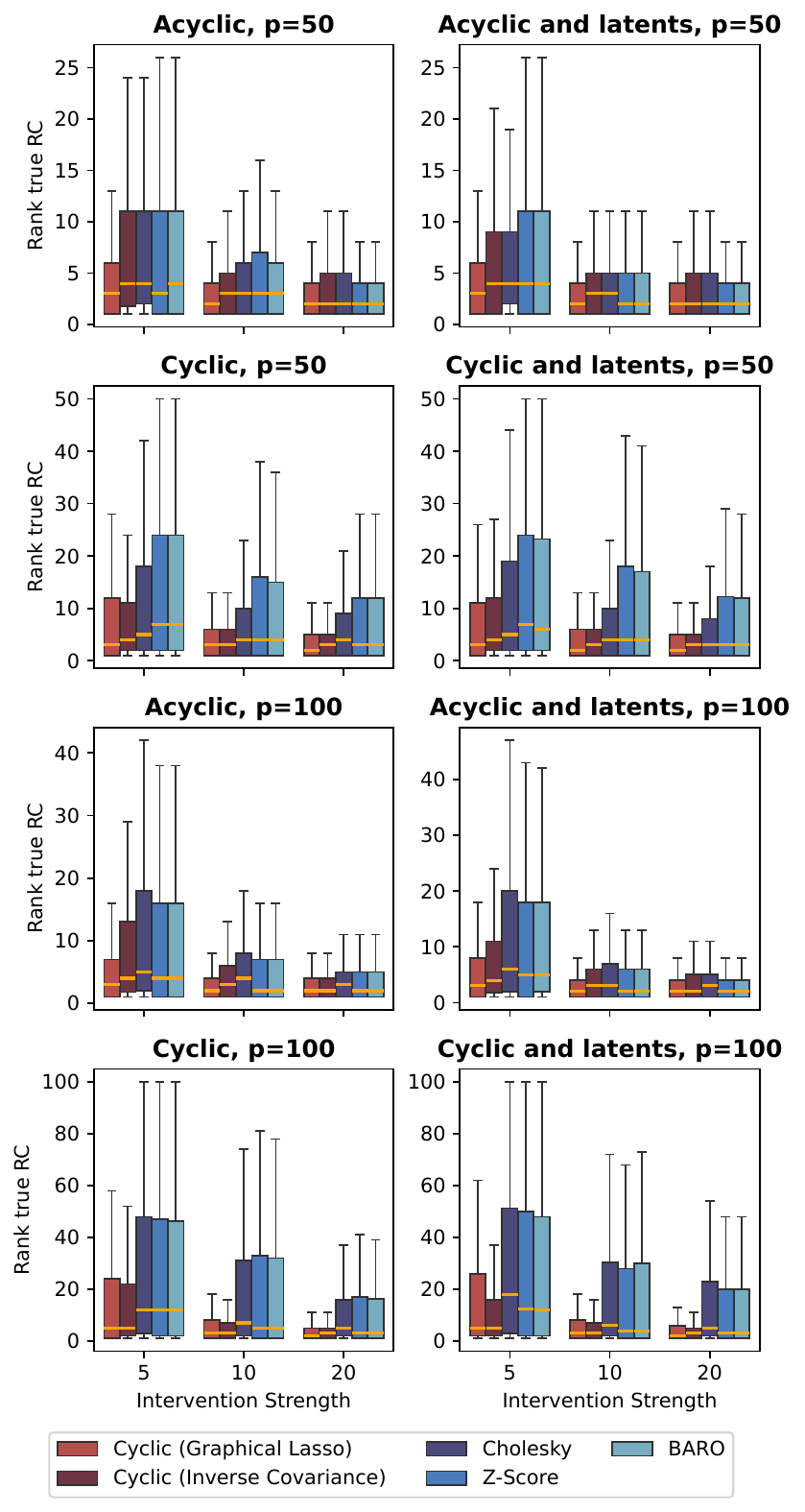}
        \caption{The boxplots illustrate the rank of the true root cause score among the scores of all nodes. Cyclic (Graphical Lasso) achieves the strongest performance, followed by Cyclic (Inverse Covariance). Cholesky performs similarly to Cyclic for the acyclic settings with $p=50$, but worse for the remaining settings. 
        Z-Score and BARO yield weaker results, especially in cyclic settings, probably because they simply treat the most anomalous node as the root cause.}
\label{fig:simulation_studies}
\end{figure}

\begin{table}[]
\centering
\caption{Means and standard deviations of the compute times in seconds for all methods. Cyclic (G) indicates Cyclic (Graphical Lasso) and (I) Inverse Covariance. }
        \begin{tabular}{lll}
%\toprule
 \textbf{METHOD} & \multicolumn{2}{l}{\textbf{COMPUTE TIME}} \\ 
  & \multicolumn{1}{l}{$\boldsymbol{p=50}$} & \multicolumn{1}{l}{$\boldsymbol{p=100}$}\\
 \hline \\
Cyclic (G) & \phantom{116}2.82±1.96 & \phantom{11}781.93±463.66 \\
Cyclic (I) & \phantom{116}0.13±0.04 & \phantom{1185}1.64±0.36 \\
Cholesky & 1168.04±1078.36 & 11853.67±10384.63 \\
Z-Score & \phantom{116}0.00±0.01 & \phantom{1185}0.01±0.03 \\
BARO & \phantom{116}0.08±0.03  & \phantom{1185}0.30±0.17 \\
%\bottomrule
\end{tabular}
    \label{tab:compute_times}
\end{table}

We assess our method by comparing it against the Cholesky-decomposition-based approach introduced in \cite{Li2024} and the BARO method proposed in \cite{pham2024BARO}, which achieved the strongest performance in the benchmark study of \cite{pham2024root}. The latter evaluates 21 RCA methods with respect to their effectiveness in microservice environments.
In addition, as in the study of \cite{Li2024}, we consider a Z-Score baseline that directly employs $\text{Z-Score}(x_i; \boldsymbol{X}_i)^2$ as the root cause score.\footnote{Both, BARO and the Z-Score baseline simply rank nodes by their anomalousness. This approach offers no theoretical guarantees for detecting the actual root cause except for the true graph being a collider-free polytrees. Nevertheless, some justification for inferring highly scored nodes as root cause candidates has been provided by \cite{RCA_graph_okati2024rca} when scores are obtained from p-values.}
For the Graphical Lasso variant of our algorithm, we fix the regularization parameter to $0.1$; additional results for different values of the parameter are provided in Appendix Figure~\ref{fig:tune_lambda}. %We use the standard version of the Cholesky algorithm if the number of features is below $100$, and the highdimensional otherwise.

To generate data, we first simulate graphs with 50 or 100 observed nodes across four settings: cyclic or acyclic, each either with or without five additional latent nodes. Edges are sampled independently at random with a probability calibrated such that the expected node degree is $3$. Given the graph, data with $m=10\cdot p$ is synthetically generated via the corresponding linear SEM, where the non-zero entries of the coefficient matrix $B$ are uniformly sampled from $[-2, -0.5] \cup [0.5, 2]$. Noise follows one of four distributions (Gaussian, uniform, exponential, lognormal). The number of root causes varies between 1 and 4, and intervention strength $\delta$ ranges from $5$ to $20$. 

Figure \ref{fig:simulation_studies} reports the aggregated results across all noise distributions and numbers of root causes. Each box summarizes 400 replications with varying random seeds. The median is indicated by the yellow line, the box represents the interquartile range (IQR), and the whiskers extend to 1.5 times the IQR. Table \ref{tab:compute_times} shows the corresponding compute times. 
Both Cyclic methods match Cholesky in accuracy for $p=50$ in the acyclic case, but are substantially faster and also more accurate for $p=100$ and all cyclic cases.
Detailed breakdowns by noise type and number of root causes, as well as experiments with varying expected node degree, are deferred to the appendix.
\subsection{Real-World Experiemnts}

\begin{table}[t]
    \centering
    \caption{Overview of the number of features  ($p$), observational samples ($m_\text{obs}$), interventional samples ($m_\text{int}$), and number of interventional datasets ($n_\text{int}$) per benchmark. The datasets occur in the same order as in Figure \ref{fig:real_world}. Differences in the number of features stem entirely from preprocessing, while variations in sample sizes are largely inherited from the original datasets.}
    \begin{tabular}{lrrrr}
%\toprule
 %\multicolumn{5}{l}{\textbf{DATA CHARACTERISTICS}} \\
 & $\boldsymbol{p}$ & $\boldsymbol{m_\text{obs}}$ & $\boldsymbol{m_\text{int}}$ & $\boldsymbol{n_\text{int}}$\\
 \hline \\
%\midrule
C10 & 10 & 3000 -- 3060 & 940 -- 1000 & 100 \\
C50 & 50 & 3000 -- 3060 & 940 -- 1000 & 100 \\
R10 & 10 & 2000 -- 2060 & 1940 -- 2000 & 100 \\
R50 & 50 & 2000 -- 2060 & 1940 -- 2000 & 100 \\
S1 & 28 -- 31 & 265 -- 358 & 209 -- 297 & 50 \\
S2 & 225 -- 277 & 360 -- 420 & 187 -- 361 & 125 \\
OB & 44 -- 56 & 360 -- 2160 & 301 -- 2101 & 125 \\
TT & 533 -- 828 & 360 -- 540 & 301 -- 481 & 125 \\
GE & 19737 & 364 & 1 & 57 \\
PS & 10 -- 20 & 542 -- 1575 & 4 -- 5 & 24 \\
%\bottomrule
\end{tabular}
    \label{tab:microservice_datasets}
\end{table}
\begin{figure}[!thb]
    \centering
     % 0.98
    \includegraphics[width=1\linewidth]{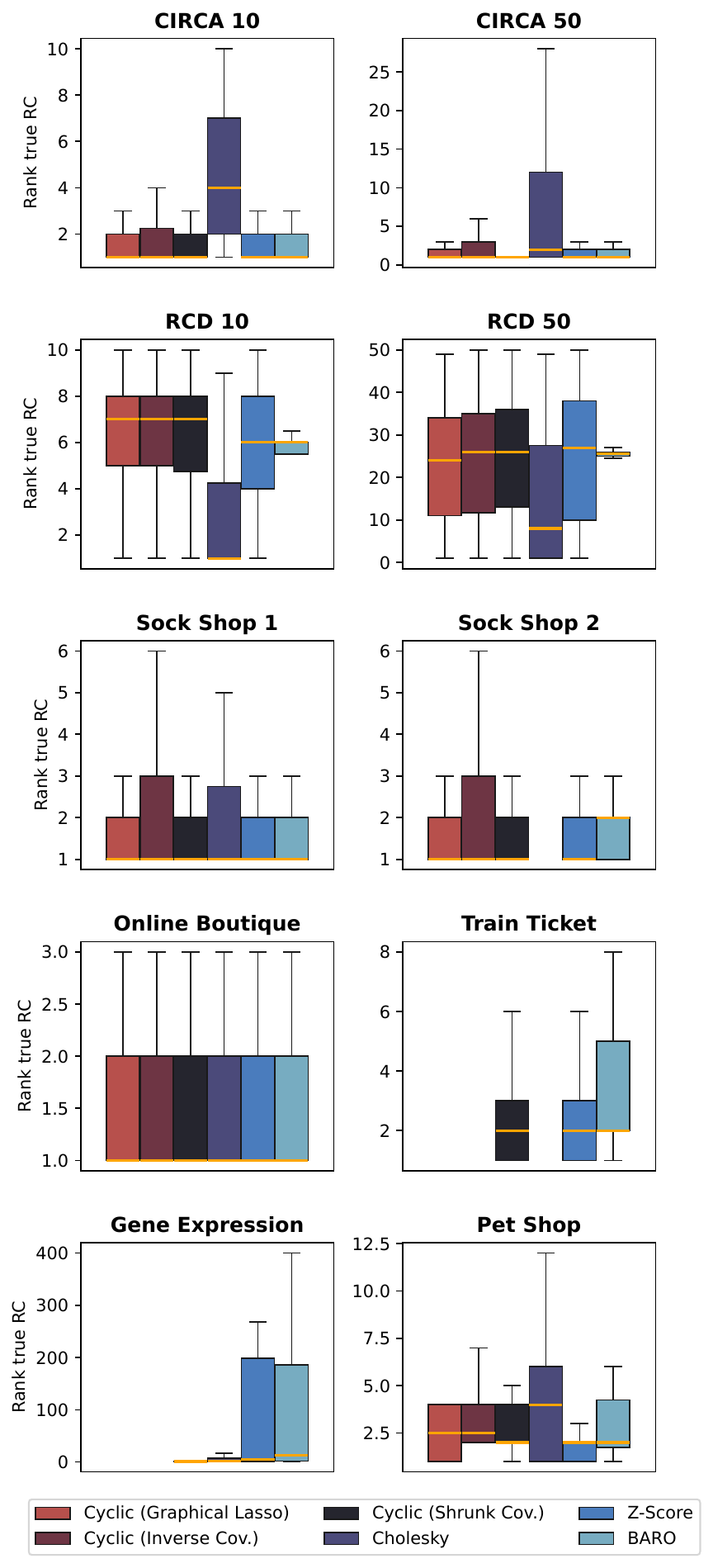}
    \caption{Results on the real-world datasets. All methods struggle with the discrete RCD dataset, although Cholesky still performs best. In contrast, on the continuous CIRCA data, Cholesky performs worse. On the gene expression data, the Cyclic and Cholesky methods clearly outperform the anomaly scorers. Otherwise, performance is largely similar across methods.}
    \label{fig:real_world}
\end{figure}
For a case study on real-world data, we use three benchmark collections. Each benchmark consists of one or more datasets from the normal mode alongside data from several incidents.
In contrast to the simulation setup described above, often, multiple interventional samples are available.

For instance, the PetShop benchmark \citep{hardt2024petshop} contains latency, request, and availability metrics for 41 services in a microservice-based application that mimics a pet adoption platform. The metrics are collected at 5-minute intervals. In addition to the normal data, the dataset includes several injected performance issues affecting different services. Each issue lasts approximately half an hour, resulting in $5$ interventional samples per incident.
The dataset contains substantial missing data and exhibits strong collinearity among features. We therefore preprocess the data, which reduces the number of features to less than half of the original set; see the appendix for details. An overview of the sample and feature sizes after the preprocessing  is provided in Table \ref{tab:microservice_datasets}.

In addition, following \cite{Li2024}, we use the gene expression dataset collected by \cite{Yepez_2022}. The dataset contains gene expression measurements from 57 patients with a Mendelian disease, together with samples from healthy individuals that serve as a baseline. The baseline data is high-dimensional, with substantially more features than observations. Consequently, neither Graphical Lasso nor empirical covariance estimation can be applied for precision matrix estimation. Instead, we employ the Shrunk Covariance estimator. Also, the standard Cholesky algorithm is no longer applicable, so we rely on their high-dimensional version \cite[Algorithm 3)]{Li2024}.

Finally, we evaluate our approach on the RCAEval benchmark \citep{pham2024root}, which comprises eight datasets. Four of these are synthetic benchmarks (RCD 10, RCD 50, CIRCA 10, CIRCA 50) generated from random directed acyclic graphs. The first two datasets contain continuous variables, whereas the latter two consist of discrete variables; see \cite{pham2024root} for details.
The remaining datasets consist of metrics collected from real-world microservice applications at one-second intervals. These include the e-commerce system Sock Shop (in two variants), the Online Boutique application with 12 services supporting browsing and purchasing, and Train Ticket, a large-scale train booking system comprising 64 services. Compared to the other systems, Train Ticket exhibits longer, more complex failure propagation paths. Since multiple metrics are recorded per service, the number of features $p$ exceeds the number of services. For each application, different types of incidents, such as high CPU load, are injected at various services. Cholesky is omitted for Sock Shop 2 and Train Ticket due to excessive runtime. For Train Ticket and Gene Expression, precision matrices are estimated exclusively using the Shrunk Covariance estimator, as the number of features exceeds the number of samples. 

Figure \ref{fig:real_world} shows the rank of the true root causes across all benchmark datasets. Following \cite{pham2024BARO}, when multiple interventional samples are present, we first compute a score for each sample individually and then use the maximum of these scores. Since this approach is computationally prohibitive for the Cholesky method, there, the interventional samples are aggregated using the mean beforehand.

In the appendix, for the RCAEval benchmark, we include a more detailed breakdown of the results, distinguishing the different fault types and  reporting the Avg@5 metric rather than  the rank of the true root cause. Doing so, we mimic the setup in \cite{pham2024root} and show that our method outperforms 18 of the 21 methods considered there. We also present experiments on semi-synthetic data illustrating that our method remains effective under moderate violations of the linearity assumption. Finally, we compare our approach with the ICE method \cite{Rothenhaeulser2015}, since it can identify intervention targets even for cyclic graphs. However, it requires data from at least three distinct environments to be available.

\section{CONCLUSION AND LIMITATIONS}{
We introduce a simple yet effective approach for RCA in linear SEMs, which, to our knowledge, is the only method that provably 
works for unknown cyclic graphs using a single anomalous sample. In addition, it remains applicable 
in the presence of multiple root causes or latent variables, given that the graph is sufficiently sparse. It relies crucially on the assumption that, except for the root causes, the SEM describing the normal regime also holds for the anomalous regime. 
Whether this assumption holds depends heavily on the respective use case; we have shown a real-world example where the method at least clearly outperforms chance level. In simulation studies, our method outperforms existing approaches while offering faster computation.} 
There are two complementary ways to justify the assumption that only one or a few structural equations change in the anomalous regime while the others remain unchanged.
On the one hand, one may argue that the task of root cause analysis would otherwise be ill-defined, since all nodes could be considered root causes \citep{Budhathoki2022}. On the other hand, the {\it Independence of Mechanisms Principle} (see Sections 2.1 and 2.2 in \cite{causality_book}) posits that mechanisms vary independently, making it unlikely that {\it rare} changes occur simultaneously at the same time stamp or for the same statistical unit. Moreover, the {\it Sparse Mechanism Shift Hypothesis} \citep{Schölkopf} explicitly assumes that distributional shifts can be attributed to changes in only a small number of mechanisms.%}

\subsubsection*{Acknowledgements}
We thank the reviewers for the insightful discussions and helpful feedback, which substantially improved this paper. This project has received funding from the European Research Council (ERC) under the European Union's Horizon 2020 research and innovation programme (grant agreement No 883818). 
DS acknowledges support by the DAAD program Konrad Zuse Schools of Excellence in Artificial Intelligence, sponsored by the Federal Ministry of Education and Research.
\FloatBarrier
\bibliographystyle{abbrvnat}
\bibliography{references}

%%%%%%%%%%%%%%%%%%%%%%%%%%%%%%%%%%%%%%%%%%%%%%%%%%%%%%%%%%%%
\section*{Checklist}

\begin{enumerate}

  \item For all models and algorithms presented, check if you include:
  \begin{enumerate}
    \item A clear description of the mathematical setting, assumptions, algorithm, and/or model. Yes, all assumptions are clearly stated alongside the theorems, the model is described  `Notation and Model' section and the algorithm in Algorithm \ref{algorithm} and the subsequent text.
    \item An analysis of the properties and complexity (time, space, sample size) of any algorithm. Yes, in Section~\ref{sec:implementation}, we derive a Type I error control guarantee as well as computational and memory complexity. Moreover, our simulations suggest that the algorithm performs reliably with sample sizes on the order of ten times the number of features.
    \item (Optional) Anonymized source code, with specification of all dependencies, including external libraries. Yes, source code is publicly available.
  \end{enumerate}

  \item For any theoretical claim, check if you include:
  \begin{enumerate}
    \item Statements of the full set of assumptions of all theoretical results. Yes, the full set of assumptions is given alongside the theorems, and if applicable, the details on the genericity assumption required is given in the supplement.
    \item Complete proofs of all theoretical results. Yes, all proofs can be found either in the main paper or in the supplement.
    \item Clear explanations of any assumptions. Yes, all assumptions are clearly explained. Our most involved assumption is the genericity assumption, which is introduced in the   `Notation and Model' section, motivated by an example and the precise exception sets are given in the appendix.
  \end{enumerate}

  \item For all figures and tables that present empirical results, check if you include:
  \begin{enumerate}
    \item The code, data, and instructions needed to reproduce the main experimental results (either in the supplemental material or as a URL). Yes, the code to reproduce our plots and experimental results is attached in the supplemental material.
    \item All the training details. Yes, while we do not train a model, all hyperparameters and all preprocessing details are clearly explained in \ref{sec:experiments} and the appendix.
    \item A clear definition of the specific measure or statistics and error bars (e.g., with respect to the random seed after running experiments multiple times). Yes, a clear definition is given in Section \ref{sec:experiments}.
    \item A description of the computing infrastructure used.  Yes, the computing infrastructure used is described in the beginning of Section \ref{sec:experiments}
  \end{enumerate}

  \item If you are using existing assets (e.g., code, data, models) or curating/releasing new assets, check if you include:
  \begin{enumerate}
    \item Citations of the creator if your work uses existing assets. Yes, we include the citation to all existing code and data used (PetShop, RCAEval and gene expression datasets, BARO method, Cholesky method, Graphical Lasso; dowhy package in the appendix).
    \item The license information of the assets, if applicable. Yes, we include the license information of our own code as well as of all existing assets we use.
    \item New assets either in the supplemental material or as a URL, if applicable. Yes, we include all new assets in the supplemental material.
    \item Information about consent from data providers/curators. Yes, the licenses of all used assets are included in our code repository.
    \item Discussion of sensible content if applicable, e.g., personally identifiable information or offensive content. Not Applicable.
  \end{enumerate}

  \item If you used crowdsourcing or conducted research with human subjects, check if you include:
  \begin{enumerate}
    \item The full text of instructions given to participants and screenshots. Not Applicable.
    \item Descriptions of potential participant risks, with links to Institutional Review Board (IRB) approvals if applicable. Not Applicable.
    \item The estimated hourly wage paid to participants and the total amount spent on participant compensation. Not Applicable.
  \end{enumerate}

\end{enumerate}

\clearpage
\appendix
\thispagestyle{empty}

% Supplementary material: To improve readability, you must use a single-column format for the supplementary material.
\onecolumn
\aistatstitle{Supplementary Material}

\section{MISSING PROOFS AND DETAILS ON GENERICITY ASSUMPTIONS}

\subsection{Proof of Lemma 2.1}
We begin by expressing the inverse in terms of the adjugate:
 $$\left((I- A)^{-1}\right)_{ij} = \frac{1}{\det(I-A)} (\text{adj}(I-A))_{ij}.$$ 
To expand the adjugate entry $\text{adj}(I - A)_{ij}$, fix $i$ and $j$, and let $I - \tilde{A}$ the matrix obtained from  $I-A$ by setting all entries in column $i$ and row $j$ to zero, except $(I-\tilde{A})_{ji}$, which is set to 1. Then, by Leibniz-Expansion,
    \begin{align}\label{eq:leibniz_expansion}
        (\text{adj}(I-A))_{ij} &= \det(I-\tilde{A}) = \sum_{\sigma \in S_{p}} \left(\text{sgn}(\sigma)\prod_{m=1}^{p}(I-\tilde{A})_{m, \sigma(m)}\right)
    \end{align}
Note that any permutation can be decomposed into disjoint cyclic components\footnote{A cyclic component of a permutation $\sigma \in S_n$ is a sequence of distinct elements $C = (i_1\ i_2\ \dots\ i_k)$ such that
$\sigma(i_1) = i_2,\ \sigma(i_2) = i_3,\ \dots,\ \sigma(i_{l-1}) = i_l,\ \sigma(i_l) = i_1$. It can be viewed as an element of $S_n$ by defining $C(j)=j$ for all $j\in [p]\setminus\{i_1, \dots, i_l\}$. Using this extension, one obtains that $\sigma$ is the product of all its cyclic components.}, say $\sigma=C_1\circ \dots \circ C_s.$ These components can be understood as graphical cycles: The matrix $I - \tilde{A}$ is the adjacency matrix of the graph $\tilde{G}$ constructed from $G$ with these modifications:
\begin{enumerate}
    \item Remove all incoming edges to node $j$ and all outgoing edges from node $i$.
    \item Multiply all edge weights by $-1$.
    \item Add self-loops at each node except for $i, j$.
    \item Add a single edge $i\to j$ with weight 1.
\end{enumerate}
In \eqref{eq:leibniz_expansion}, the summand for a permutation $\sigma \in S_n$ is non-zero only if for all $k$, the edge $k \to \sigma(k)$ exists in $\tilde{G}$.  Therefore, we restrict to these permutations henceforth. For such a permutation $\sigma$, each cyclic component $\sigma(i_1) = i_2,\ \sigma(i_2) = i_3,\ \dots,\ \sigma(i_{k-1}) = i_k,\ \sigma(i_k) = i_1$ corresponds to the cycle  $i_1 \to i_k \to i_{k-1} \to \dots \to i_2 \to i_1 \in \tilde{G}$. Also the sign of the permutation can be decomposed according to the cycles: $\text{sgn}(\sigma) = (-1)^{p - s} = (-1)^{l_1 + \dots + l_s - s}$, $l_*$ being the lengths of the cycles. Combined, this leads to
\begin{align}\label{eq:adj_graphical_cycles}
        (\text{adj}(I-A))_{ij} = \sum_{\substack{(C_1, \dots, C_s): \text{ disjoint cycles}\\ \text{in $\tilde{G}$ covering all nodes $[p]$.}}} (-1)^{s} \prod_{m=1}^s (-1)^{l_m} \prod_{k\to h\in C_m}  (I-\tilde{A})_{hk}.
    \end{align}
To relate this back to $G$, we go through the differences with $\tilde{G}$. Fix a permutation $\sigma =  C_1 \circ \cdots \circ C_s$ contributing a non-zero summand. First note that due to 1. and 4. the only outgoing edge of $j$ leads to $i$, so $\sigma(j)=i$; otherwise the corresponding summand would be 0. In graphical terms, this means $i\to j$ is in one of the cycles, lets say in $C_s$. Because of 4., $i\to j$ is not necessarily in $G$. However, it has anyways weight 1 in $\tilde{G}$, so we may omit this factor in the product; in formulas
$$ (-1)^{l_s} \prod_{h\to k\in C_s}(I-\tilde{A})_{hk} = (-1)^{l_{\pi}+1}\prod_{h\to k \in \pi}(I-\tilde{A})_{hk}$$
where $\pi$ is the path $j \rightsquigarrow i$ obtained from $C_s$ by dropping $i \to j$, and $l_{\pi}$ its length. Note that $\pi$ does not feature any self-loops, since the cyclic component $C_s$ consists of distinct elements. Therefore, $\pi \in G$, and the expression further simplifies to
\begin{equation}\label{eq:simplification_path}
\begin{aligned}
  (-1)^{l_{\pi}+1}\prod_{h\to k \in \pi}(I-\tilde{A})_{hk} \overset{h \neq k}{=} &(-1)^{l_\pi} \prod_{k\to h\in \pi}  (\tilde{A})_{hk} \\ \overset{2.}{=}\  &(-1)^{l_{\pi}+1} \prod_{k\to h\in \pi}  -a_{hk} \\ =\ &-\prod_{k\to h\in \pi} a_{hk}.
\end{aligned}
\end{equation}
Analogously, for every  cycle neither using $i\to j$ nor self-loops,$$(-1)^{l_m} \prod_{k\to h\in C_m}  (I-\tilde{A})_{hk}  = \prod_{k\to h\in C} a_{hk}.$$
Finally, some of the cycles might be self-loops, lets say $C_1, \dots, C_r$. Again, they have edge weight 1, and also the signs $(-1)^{r}\cdot\prod_{m=1}^r (-1)^{l_m}=1$ cancel out, so we may omit them from the product in \eqref{eq:adj_graphical_cycles}. This omission is reflected in the overall sum by no longer requiring that $(C_1, \dots, C_{s-1}, \pi)$ together cover all nodes. More precisely, define $\mathcal{C_\pi} = \{(C_1, \dots, C_q): \text{collection of disjoint cycles in $\tilde{G}$ not using any node in $\pi$; $q\in \mathbb{N}_0$}.\}$. Then,
\begin{align*}
        (\text{adj}(I-A))_{ij}  = -\sum_{\pi: j \rightsquigarrow i}
        \prod_{k\to h\in \pi}  a_{hk} \cdot \left(
        \sum_{(C_1, \dots, C_q) \in \mathcal{C}_\pi} (-1)^{q+1}\prod_{m =1}^q  \prod_{k\to h\in C_q}  a_{hk}\right).
    \end{align*}
In this expression, the term $(-1)^s$ in \eqref{eq:adj_graphical_cycles} becomes  $(-1)^{q+1}$ since $\pi$ is extracted from the overall sum, therefore not counted by $q$. Multiplying with $-1$ and moving the summand for $q=0$ outside of the sum, we obtain the formula stated in the lemma:
\begin{align*}
        ((I-A)^{-1})_{ij}  = \frac{1}{\det(I-A)}  \sum_{\pi: j \rightsquigarrow i}
        \prod_{k\to h\in \pi}  a_{hk} \cdot \left(1+
        \sum_{\substack{(C_1, \dots, C_q)\in \mathcal{C}_\pi,\\q\geq 1.}} (-1)^{q} \prod_{m =1}^q  \prod_{k\to h\in C_q}  a_{hk}\right).
    \end{align*}
 %Its sign is determined by (-1)^{\# self-loops} \cdot (-1)^{\# self-loops + 1} = 1 (first factor as the product over all self-loops of (-1)^1, second one corresponfs to the q+1).
%  Now whenever (i) either $k=j$ or $\sigma(k)=i$ (but not both) , $(I-\tilde{A})_{k, \sigma(k)}$ is zero. Moreover, whenever (ii) $\sigma(k) 
% \to k \not\in G$ also zero. So, we can restrict to permutations which have neither of these things met. This means, for sure $\sigma(j)=i$ because of $(i)$. 
% Note that $\sigma$ can be decomposed into disjoint cycles (in the permutation sense), and also $\text{sgn}$.... 
% Because of (ii) these now correspond to actual cycles in the graph, except for one of these which uses $j\to i$, or self-loops.

% For the signs note that we have $(-1)^{m_1+\dots+m_s-s}$
% At each single cycle note that $(-1)^{m_1}$ cancels out with the $-$ in $I-A$ (or more precisely 
% $(-1)^{m_1} \prod (-a_{hl}) = \prod (a_{hl})$

\subsection{Proof of Theorem 3.1}
As shown in Equation \eqref{eq:xi} in the main paper, $\xi \stackrel{d}{=} \Xi + ((I - A)^T  \Theta_{NN} \Delta)$, implying that $\xi_i \centernot{\stackrel{d}{=}} \Xi_i$ if and only if $((I - A)^T  \Theta_{NN} \Delta)_i \centernot{=} 0$. Using that $\Delta$ is non-zero only at the root cause positions, we obtain
$$((I - A)^T  \Theta_{NN} \Delta)_i   = ((I - A)_{:,i})^T  \Theta_{NN} \Delta  =  \sum_{r\in\mathcal{R}} (I - A)_{r,i}\sigma_r^{-2}\delta_r = \sigma_i^{-2}\delta_i -\sum_{r\in\mathcal{R}\cap\text{pa}(i)} a_{r, i}{\sigma_r^{-2}}\delta_r,$$ which is generically non-zero for all nodes except the root causes and their parents, and definitively zero for all other nodes. The exceptional parameter set on which $\xi_r \stackrel{d}{=}  \Xi_r$ for some root cause $r$
can be directly read off from the above equation: it consists of all parameter choices $a_{ij}, \sigma_{N_i}, \delta_i$ for which, for any root cause $r \in \mathcal{R}$, \[((I - A)^T  \Theta_{NN} \Delta)_r=0.\]That is, when the causal model is applied to $\Delta$, the perturbations cancel exactly. This situation only arises if there are multiple root causes that are directly linked in the graph for a measure zero set of choices for all non-zero $a_{ij}, \sigma_{N_i}, \delta_i$ .

\subsection{Proof of Theorem 3.2} The claim follows directly from Theorem 3.1 together with the observation that root causes and their descendants are generically anomalous, i.e., $\tilde{X}_r \centernot{\stackrel{d}{=}} X_r$, while for all other nodes $\tilde{X}_i {\stackrel{d}{=}} X_i$. The exceptional set is the union of that from Theorem 3.1 and all parameter choices for which even a root cause is not anomalous, namely those satisfying \[0 = ((I - A)^{-1}  \Delta)_r.\] Again, this can only happen if there are multiple directly connected root causes.

\subsection{Proof of Lemma 3.4} Assume first that the noise vector $\bN$ is jointly Gaussian, then $\bar{\bN} = \bN_{\bX} +(I-A)_{\bX\bL}(I-A_{\bL\bL})^{-1}\bN_{\bL} $ is also Gaussian, and the sparsity pattern of the precision matrix $K_{\bar{\bN}} := \Sigma_{\bar{\bN}}^{-1}$ corresponds to conditional independence  relations among the components of $\bar{\bN}$. In particular, we have
\[
(K_{\bar{\bN}})_{ij} = 0 \Leftrightarrow \bar{\bN}_i \indep \bar{\bN}_i \mid \bN_{[p]\setminus \{i,j\}}.
\]
Denoting $B = (I-A)_{\bX\bL}(I-A_{\bL\bL})^{-1}$, we observe that  $\bar{N}$ combined with $\bN_{\bL}$ fulfills the linear SEM \begin{equation}\label{eq:SEM_proof_precision_lemma}
    \begin{pmatrix}
    \bar{\bN} \\
    \bN_{\bL}
\end{pmatrix} = \begin{pmatrix}
   0 & B \\ 0 & 0
\end{pmatrix}\begin{pmatrix}
    \bar{\bN} \\
    \bN_{\bL}
\end{pmatrix}+\begin{pmatrix}
    \bN_{\bX} \\
    \bN_{\bL}
\end{pmatrix} 
\end{equation}
 Conditional independencies among the components of $\bar{\bN}$ then follow from d-separation in the corresponding graph. Specifically, $
\bar{\bN}_i \not\indep \bar{\bN}_i \mid \bN_{[p]\setminus \{i,j\}}.
$ is equivalent so $\bar{\bN}_i$ and $\bar{\bN}_j$ being d-connected given $\bar{\bN}$. Since the SEM \eqref{eq:SEM_proof_precision_lemma} only features edges from $\bN_{\bL}$ to $\bar{\bN}$ this further breaks down to the existence of a path $\bar{\bN}_i \leftarrow \bN_{\bL_{l_1}} \to \bar{\bN}_{i_1} \leftarrow \bN_{\bL_{l_2}} \to  \cdots \to \bar{\bN}_j.$ Since $B=-A_{\bX\bL}(I-A_{\bL\bL})^{-1}$ and $(I-A_{\bL\bL})^{-1}$ encodes directed paths using only latents,  $B_{il}\neq 0$ whenever in the graph belonging to $\boldsymbol{Z}$, there is a path from $l$ to $i$ using just latents. Combining both, the statement of the lemma follows.

Turning to general distributions, note that $K_{\bar{\bN}}= (\Sigma_{\bN_{\bX}}+B\Sigma_{\bN_{\bL}}B^T)^{-1}$ which depends only on the covariance of $\bN$ and no other properties of the distribution. Thus, the lemma holds for arbitrary distributions of $\bN$.

\subsection{Proof of Theorem 3.5}
The deviation in the paper already proves that the theorem holds generically. The exception set can be obtained analogously to the one of Theorem 3.1. Specifically, $\xi_r \stackrel{d}{=} \Xi_r$ for an observable root cause $r$ if \[0 \neq ((I - \bar{A})^T  \Theta_{\bar{N}\bar{N}} \bar\Delta)_r.
\] 

\section{OTHER PERSPECTIVES ON OUR SCORE}\label{sec:alternative_derivations}
\paragraph{Relation to linear regression}
As mentioned in the main paper, there are two other interpretations of our score $s = \text{Z-Score}(\xi; \Xi)$ worth mentioning. The first one is that it relates to linear regression. Specifically, consider the linear regression model 
$$X_j = \sum_{i\neq j} \alpha_iX_i + E_j$$
with the data from the normal regime $X$ as the training data set, i.e. the coefficients are learned by fitting on this distribution. Under infinite sample size, this yields 
$$\alpha_i = -\frac{\Theta_{ji}}{\Theta_{jj}}\qquad (i\neq j),$$ and a residual with constant variance $$\text{Var}(E_j)=\frac{1}{\Theta_{jj}}.$$ Next, testing the model on the anormal sample $\tilde{X}_j$, we obtain the residual
\begin{equation}\label{eq:lin_regression_at_x_tilde}
    \tilde{E}_j = \tilde{X}_j - \sum_{i\neq j} \alpha_i\tilde{X}_i = \tilde{X}_j + \sum_{i\neq j}\frac{\Theta_{ji}}{\Theta_{jj}} \tilde{X}_i = \frac{1}{\Theta_{jj}}(\Theta\tilde{X})= \frac{1}{\Theta_{jj}}\xi.
\end{equation}
Combining this with $\text{Var}(\Xi)= \Theta$ so that $\text{Z-Score}(\xi_j; \Xi_j) = (\Theta\tilde{X})_j/\sqrt{\Theta_{jj}}$, yields the following result.
\begin{lemma} Let $E_j$ be the residual of the linear equation model predicting $X_j$ in terms of $X_{-j}$ and $\tilde{E}_j$ the residual when testing the model on $\tilde{X}_j$ as defined in \eqref{eq:lin_regression_at_x_tilde}. Then,
    \begin{equation*}
      \text{Z-Score}(\xi_j; \Xi_j)  = \frac{\tilde{E}_j}{\text{std}(E_j)}.
  \end{equation*} 
\end{lemma}

Assuming $X$ to be Gaussian, this further simplifies: In this case, the linear predictor coincides with the conditional mean, that is
$$\mathbb{E}(X_j \mid X_{-j}=\tilde{X}_{-j}) = \sum_{i\neq j} \alpha_i\tilde{X}_i.$$ Moreover, the conditional variance is the residual's variance: $\text{Var}(X_j \mid X_{-j}=\tilde{X}_{-j}) = \text{Var}(E_j)$. Therefore, the $\text{Z-Score}(\xi_j; \Xi_j)$ can also be seen as a conditional Z-Score of $\tilde{X}$:
\begin{equation*}
      \text{Z-Score}(\xi_i; \Xi_i)  = \text{Z-Score}(\tilde{X}_j; X_j \mid X_{-j}=\tilde{X}_{-j}) := \frac{|\tilde{X}_j-\mathbb{E}({X_j\mid X_{{-j}}=\tilde{X}_{-j}})|}{\text{Var}({X_j\mid X_{{-j}}=\tilde{X}_{-j}})^{1/2}}.
  \end{equation*} 

\paragraph{RCA via conditional independence testing}
Many existing RCA or intervention target estimation approaches rely on conditional independence testing \citep{Kocaoglu2019_CI, Jaber2020_CI, Mooij2020_CI, ikram2022root}. To see that, we first need to introduce a different perspective on modelling perturbations on root causes. For simplicity, we present it for the case of one root cause, but the approach can also be done for multiple ones.
Specifically, one combines the normal and anomalous datasets $\boldsymbol{X} \in \mathbb{R}^{m \times p}, \tilde{X}\in \mathbb{R}^{1 \times p}$ into one joint dataset $\boldsymbol{Y}\in \mathbb{R}^{m + 1 \times (p+1)}$ by introducing a `regime indicator' \citep{Dawid2021}\footnote{Also called `F-node' \citep{Jaber2020_CI} or `context variable' in the literature \citep{Mooij2020_CI}.} variable $F$, that is, a binary feature indicating whether a sample is anormal or normal:
\begin{equation*}
    \boldsymbol{Y} = \begin{pmatrix}
        X^{(1)} &  0 \\
        \vdots &\vdots \\
        X^{(m)} &  0 \\
        \tilde{X} & 1
    \end{pmatrix}
\end{equation*}
One can then show that the corresponding $P^Y$ satisfies again linear SEM, specifically, the equations for all non-root cause features remain the same, while
\begin{align*}
    Y_r &= \sum_{Y_i \in \text{pa}(i)\setminus \{F\}} a_{ri}Y_i + \delta F  + \epsilon_r  \\
    F &= \epsilon_0\sim \text{Bernoulli}\Big(\frac{m_2}{m + m_2}\Big).
\end{align*}
So, the corresponding graph is the same as $G$ with one additional node $F$ and an edge $F \to Y_r$. Note that this approach still works without assuming linear relations, letting each node be any function of its parents and noise term. In both cases, for DAGs, d-separation and the Markov property of the graph yields
\begin{equation}\label{eq:ci_test}
    Y_j \indep F \mid Y_{-j} \iff Y_j \notin \text{pa}(Y_r)\cup \{Y_r\},
\end{equation}
 which again reveals the root cause and its parents. Since $F$ is binary, testing this independence is the same as assessing whether the two distributions where $F = 0,1$, respectively, are the same: $$H_0: X_j \mid X_{-j}   \sim \tilde{X}_j \mid \tilde{X}_{-j};$$ compare Definition 2 in \citet{Jaber2020_CI}.
Both our $\text{Z-Score}(\xi_i, \Xi_i)$ as well as the conditional Z-Score can be seen as test statistics for this hypothesis, as they become 0 under $H_0$. For linear SEMs, \cite{spirtes1994conditional} show that the d-separation based Markov property still holds in the presence of cycles, so the above conclusions remain valid. In contrast, for non-linear SEMs with cycles, d-separation must be replaced by $\sigma$-separation to recover a valid Markov property \cite{Forre2018markovpropertycyclic}. Denoting by $\text{SCC}(Y_r)$ the set of nodes lying on a directed cycle with $Y_r$, we obtain
\begin{equation*}
    Y_j \indep F \mid Y_{-j} \iff Y_j \notin \text{pa}(Y_r)\cup \text{pa}(\text{SCC}(Y_r)) \cup \text{SCC}(Y_r),
\end{equation*}
thereby enlarging the root cause candidate set.

Moreover, without any parametric assumptions, estimating the conditional Z-Score is complex, and for both scores, without any assumptions, it is involved to derive the null distribution and assess the power. 
In contrast, as demonstrated in the main paper, under a linear SEM and for sufficiently large intervention strength $\delta$, our test has statistical power even with a single anomalous sample. This reliance on just one anomalous sample constitutes the key distinction between our approach and other conditional independence–based RCA methods. For instance, while \cite{Jaber2020_CI} also employs the test in \eqref{eq:ci_test} (among others), they do not provide a practical algorithm and establish guarantees only under the assumption of an oracle CI test. Nonetheless, since \eqref{eq:ci_test} also holds in the nonlinear case, this viewpoint may serve as a starting point for developing a generalization of our method for nonlinear data.

\section{CHOLESKY DECOMPOSITION-BASED RCA AND ITS RELATION TO OUR METHOD}

The method of \cite{Li2024} assumes $G$ to be a DAG and then identifies the root cause as follows: 
Like our approach, they start with the fact that under the linear SEM,
\[\Sigma_{X X} = (I-A)^{-1}\Sigma_{NN}(I-A)^{-T}.\]
Their core idea is that if $X$ is arranged according to a topological order of the DAG, then $I-A$ is lower triangular. Therefore, also the product $L= (I-A)^{-1}\Sigma_{NN}^{1/2}$ is lower triangular and can be recovered from $\Sigma_{X X}$: Specifically, $L$ is the so-called Cholesky factor of $\Sigma_{XX}$, defined as the lower triangular matrix satisfying $LL^T = \Sigma_{XX},$ which is unique and can be computed from $\Sigma_{XX}.$ Knowing $L$, we can recover the noise up to standardization: Since $X = (I-A)^{-1}N$,
$$N_{\text{std}} = \left((I-A)^{-1}\Sigma_{NN}^{1/2}\right)^{-1}X$$ where $N_{\text{std}}$ is the standardized version of $N$. Similarly,
the perturbed noise can be recovered (up to rescaling with the normal noise's variance) as $\Sigma_{NN}^{-1/2}\tilde{N}_{\text{std}} = L^{-1}\tilde{X}$, whose only extreme entry directly identifies the root cause.

Because the correct ordering is not known,  their method iterates over all permutations $\pi$, computing the Cholesky factor $L_\pi$ of the permuted covariance and checking whether  $L_\pi^{-1}\tilde{X}$ has exactly one extreme value. Based on non-trivial linear algebra insights, they show that as soon as they found one $\pi$ yielding precisely one extreme component $i$, then $i$ is indeed the root cause. Such a $\pi$ does not necessarily represent the true causal ordering; instead, it must locally respect the structure around $i$ insofar as the parents of $i$ precede $i$, and its descendants appear after.

To compare their method to ours, note that we obtain a sparse vector via applying the transformation $\Theta_{XX}$ to the observed vector $\tilde{X}$, while they 
get a sparse vector via 
applying $L_\pi^{-1}$ to $\tilde{X}$.
Their method is more accurate in the sense that $L_\pi^{-1}\tilde{X}$ is only extreme for the root cause, while $\Theta_{XX}\tilde{X}$ only identifies a short list of root causes. However, in the acyclic case, by intersecting with the nodes that are actually anomalous, we also obtain the unique root cause; compare Theorem \ref{thm:outliers_anomalous}.

\section{FURTHER DETAILS ON EXPERIMENTS}

\subsection{Choice of hyperparameters}

\begin{figure}[h]
    \centering
    \includegraphics[width=1\linewidth]{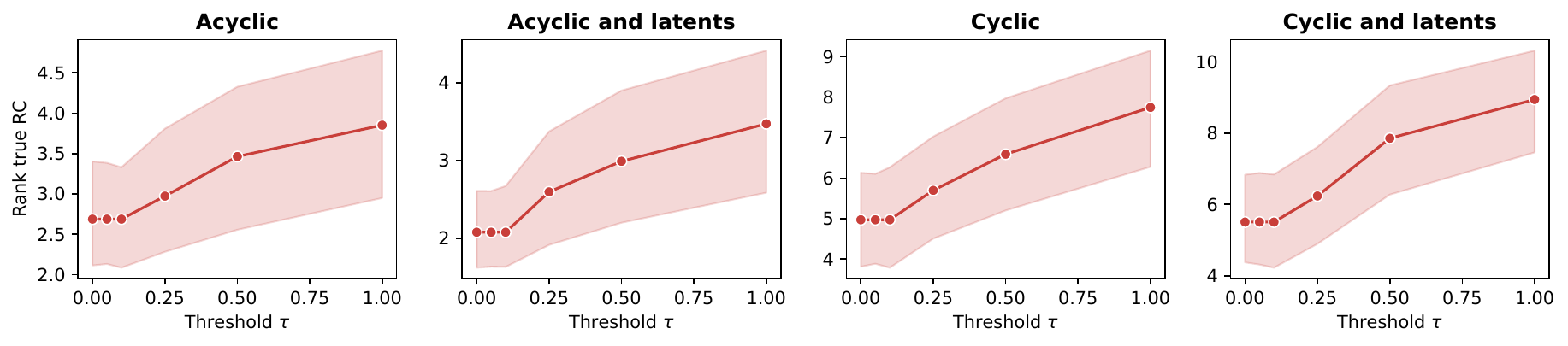}
    \caption{
        Effect of the threshold parameter $\tau$ on RCA performance when using Cyclic (Graphical Lasso) on $50$ nodes with the same data simulation setup as in the main paper. The value $\tau=0.1$ yields the best results.
    }
    \label{fig:tune_tau}
\end{figure}

Our method involves two hyperparameters: the threshold $\tau$ that distinguishes normal from anomalous features, and the regularization parameter $\alpha$ in the Graphical Lasso.  
A suitable choice of $\tau$ can be expected to depend on the strength of the intervention. Indeed, as shown in Figure~\ref{fig:tune_tau}, for lower intervention strengths it is beneficial to select a smaller $\tau$.  
This reflects the fact that an intervention strength of 5 only induces mild anomalies, since 5 is quite small compared to the noise variance in our simulation setup (for instance, for Gaussian noise, we sample the standard deviation of each noise term uniformly from $[0.5, 2]$). 
For our simulation studies, we set $\tau = 0.25$ to achieve reliable performance already at small intervention strengths. In practical applications, this threshold has the advantage of being intuitively interpretable: domain experts can often set it based on their expectation of how anomalous the data should appear in a given scenario.  
Otherwise, we recommend selecting $\tau$ after inspecting the anomaly scores of all features, potentially revealing two clusters of normal versus anormal features.

\begin{figure}[h]
    \centering
    \includegraphics[height=0.25\linewidth]{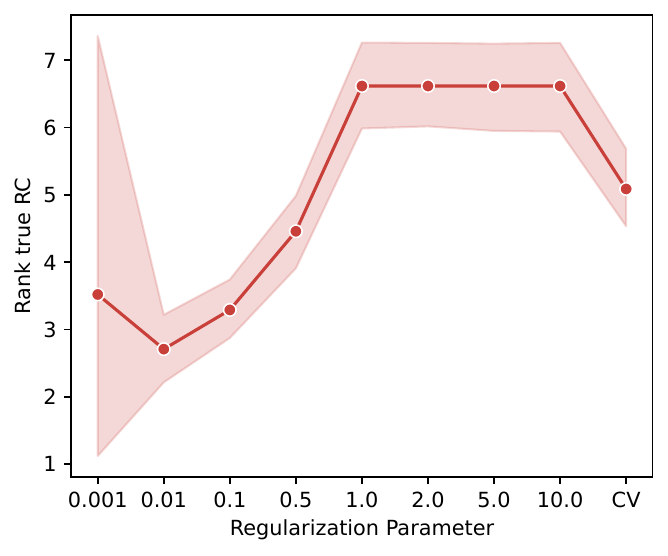}
    \includegraphics[height=0.25\linewidth]{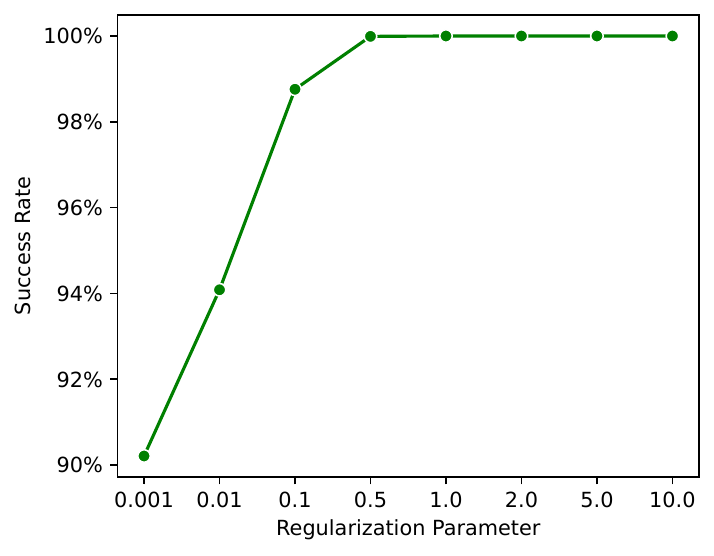}
    \caption{
        Left: Performance of Cyclic (Graphical) across different regularization parameters $\alpha$ in the Graphical Lasso. The data are generated as described in the main paper, but here all graph types and intervention strengths are combined into a single plot. Note that the Graphical Lasso may fail for small $\alpha$, since the regularization is insufficient to stabilize inversion of ill-conditioned matrices. Only successful replications are included here. `CV' denotes selection of $\alpha$ via cross-validation.
Right: Corresponding success rate of Graphical Lasso.
    }
    \label{fig:tune_lambda}
\end{figure}

Regarding the regularization parameter $\alpha$, Figure~\ref{fig:tune_lambda} (left) shows that the best results are obtained for $\alpha = 0.01$. However, as seen in the right plot, Graphical Lasso frequently fails due to instability for such small $\alpha$ values.  
Therefore, we initialize with $\alpha = 0.1$ and add $0.5$ in case of failure. If the method still fails after seven such increases, our implementation defaults to using the inverse covariance matrix instead. In practice, this fallback was never required in our experiments. Another possible strategy to ensure success is to select $\alpha$ via cross-validation. However, this approach is computationally more expensive, and also yields weaker performance, as shown in the rightmost column of Figure~\ref{fig:tune_lambda} (left).

\subsection{Precision estimation}\label{sec:precision_and_compute_time}
\begin{figure}[htbp]
% \begin{subfigure}{\linewidth}
    \centering
    \includegraphics[width=.9\linewidth]{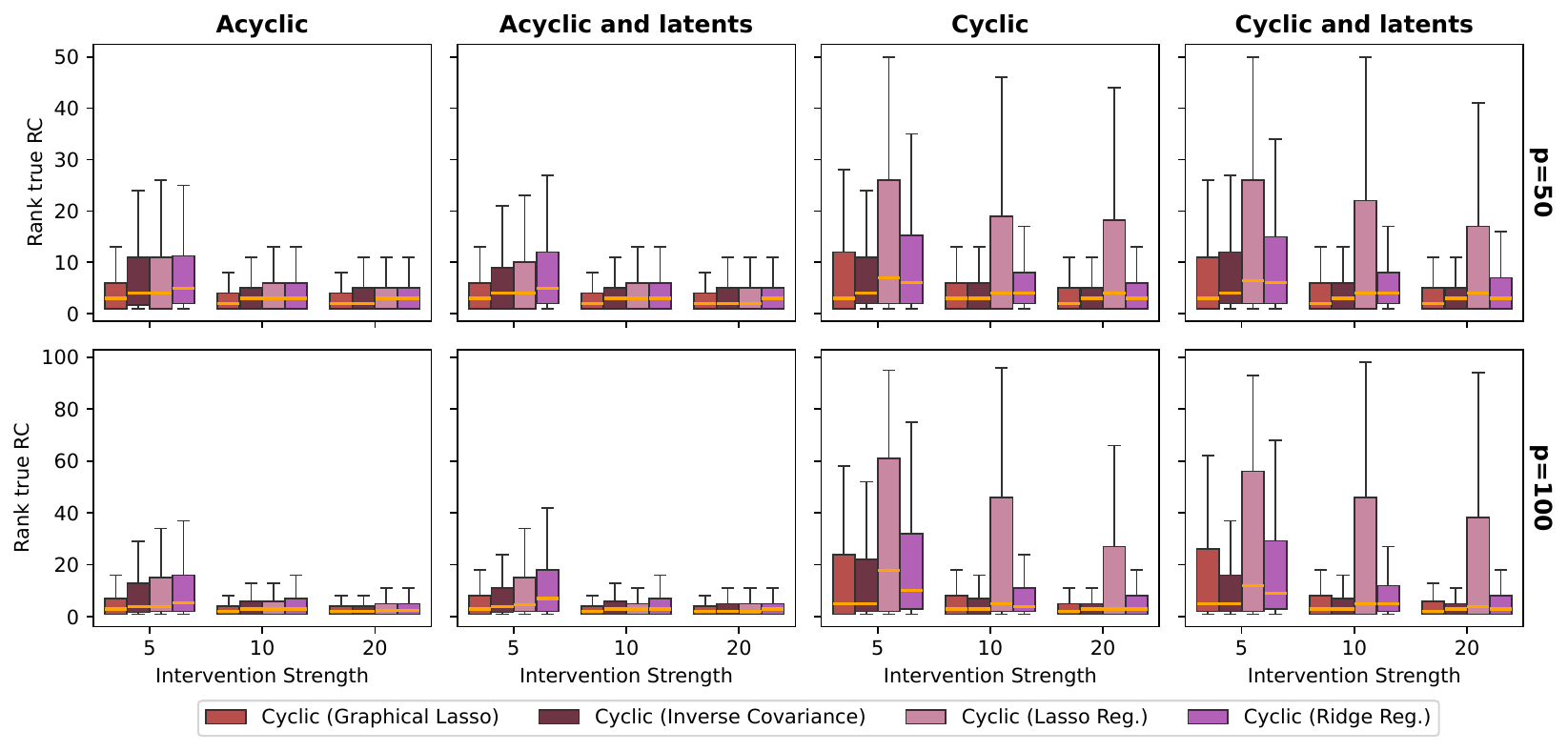}
        \caption{Ranks of the actual root causes under the same simulation settings as in Figure \ref{fig:simulation_studies}. The Lasso and Ridge regression variants both significantly underperform the methods presented in the main paper.
        }
\label{fig:alternatives_precision_estimation}
\end{figure}
A challenge in both methods is the need to estimate the inverse of the covariance matrix (our method) or its Cholesky factor (Cholesky method), which might be ill-conditioned. While stronger regularization can alleviate this issue, as shown above, it may not resolve cases such as almost-perfect feature correlation, which shows up in our real-data experiment. In this case, we recommend applying PCA to merge such highly correlated features; compare Section \ref{subsec:preprocessing}. While this merging may combine the root cause with other features, potentially obscuring its identity, we argue that if the root cause $r$ is strongly correlated with another feature $r'$, it becomes inherently difficult to distinguish between the two based on data alone. Thus, treating them as a single node is justifiable.

Even after mitigating high correlations, precision matrix estimation remains challenging and poses not only the computational but also the statistical bottleneck of our method. Therefore, in addition to the Graphical Lasso and inverting the sample covariance approaches described in the main paper, we explored alternative strategies. More specifically, we tried to bypass precision estimation entirely by first estimating the sample covariance matrix, and then solving $\hat{\Sigma}\hat{\xi} = \tilde{X}$ using Lasso or Ridge Regression with regularization parameter $10^{-4}$. In both cases, the baseline $\Xi$ is not estimated, and the scores $s_i$ in line 5 of the algorithm are simply set to $|\hat{\xi}_i|$. However, as Figure \ref{fig:alternatives_precision_estimation} illustrates, both alternatives perform worse than Graphical Lasso. The variance of the Lasso regression approach is particularly high. A likely reason is that Lasso enforces sparsity in its solutions. This means that if the root cause's score $\xi_r$ is incorrectly shrunk to zero, the corresponding root cause will be assigned the lowest possible score (i.e., 0), which is a severe misclassification. In contrast, if the other methods make statistical errors, these are less likely to produce the extreme outcome of assigning a score of 0 to the true root cause.

\subsection{Further Simulation Studies}
\paragraph{Results broken down by different noise distributions.} 

Figure \ref{fig:simulation_studies_different_noise} shows the results of the same  simulation study as in the main paper for $p=50$ but with performance metrics reported separately for each noise distribution choice.
 \begin{figure}[htbp]
    \centering
    %\begin{subfigure}{\linewidth}
        \includegraphics[width=0.9\linewidth]{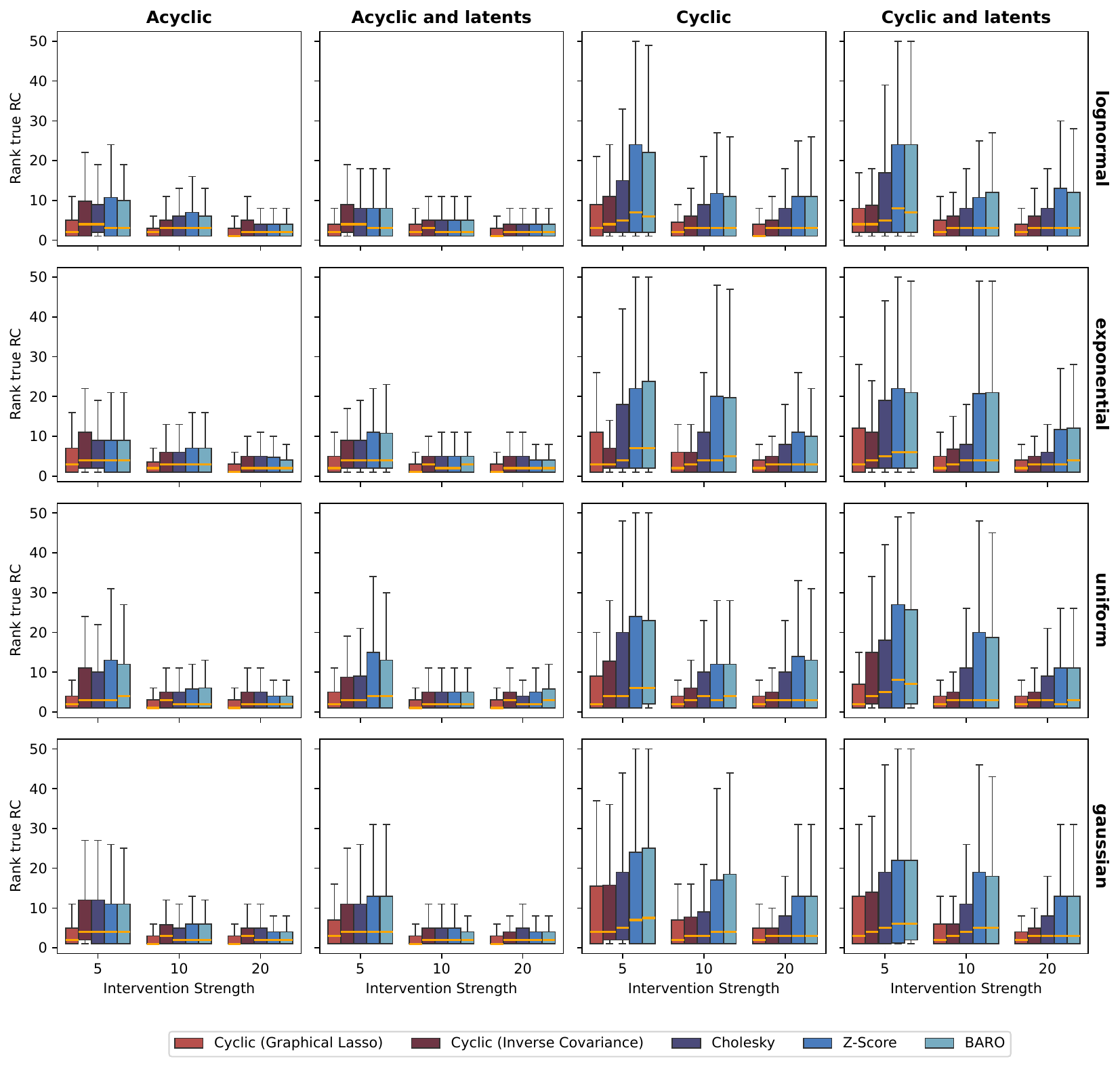}
    \caption{Performance for $p=50$ across different noise distributions. All methods perform slightly weaker under uniform noise and $\delta=5$, whereas for the other noise types and intervention strengths, no consistent patterns emerge.}\label{fig:simulation_studies_different_noise}
    \end{figure}

\paragraph{Results broken down by different numbers of root cause.} 
Figure \ref{fig:more_rcs} again visualizes the results as in the main paper for $p=50$ but broken down by how many root causes are chosen.
 \begin{figure}[htbp]
    \centering
    %\begin{subfigure}{\linewidth}
        \includegraphics[width=0.9\linewidth]{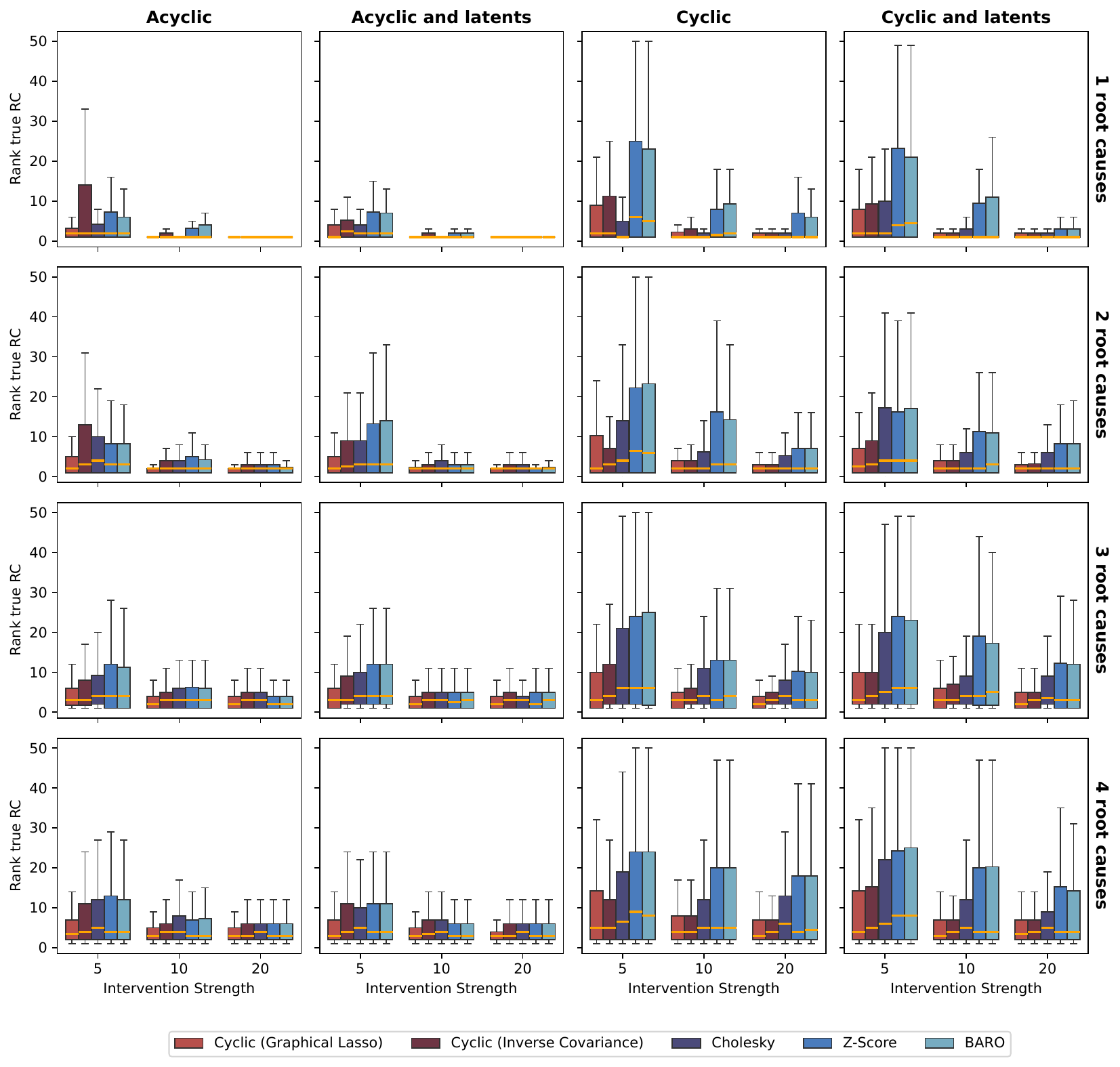}
    \caption{Performance for $p=50$ when selecting 1, 2, 3 or 4 root causes. Unsurprisingly, performance decreases as the number of root causes increases, with the largest decline observed for Cholesky, which is designed to identify only a single root cause.}\label{fig:more_rcs}
    \end{figure}

\paragraph{Varying the expected degree.} We expect the performance of both our methods as well as the BARO and Z-Score baselines to  depend heavily on the expected degree of the nodes: The higher the connectivity, the more nodes get infected by the root cause, are therefore anomalous, and  have a high Z-Score and BARO score. In addition, the increased connectivity raises the probability that a node is a parent and descendants of the true root cause at the same time, thereby appearing in the shortlist described in Theorem \ref{thm:outliers_anomalous}.
 Indeed, the performance declines with higher degree as illustrated in Figure \ref{fig:varying_edge_degree_50}.
 \begin{figure}[htbp]
% \begin{subfigure}{\linewidth}
    \centering
    \includegraphics[width=0.9\linewidth]{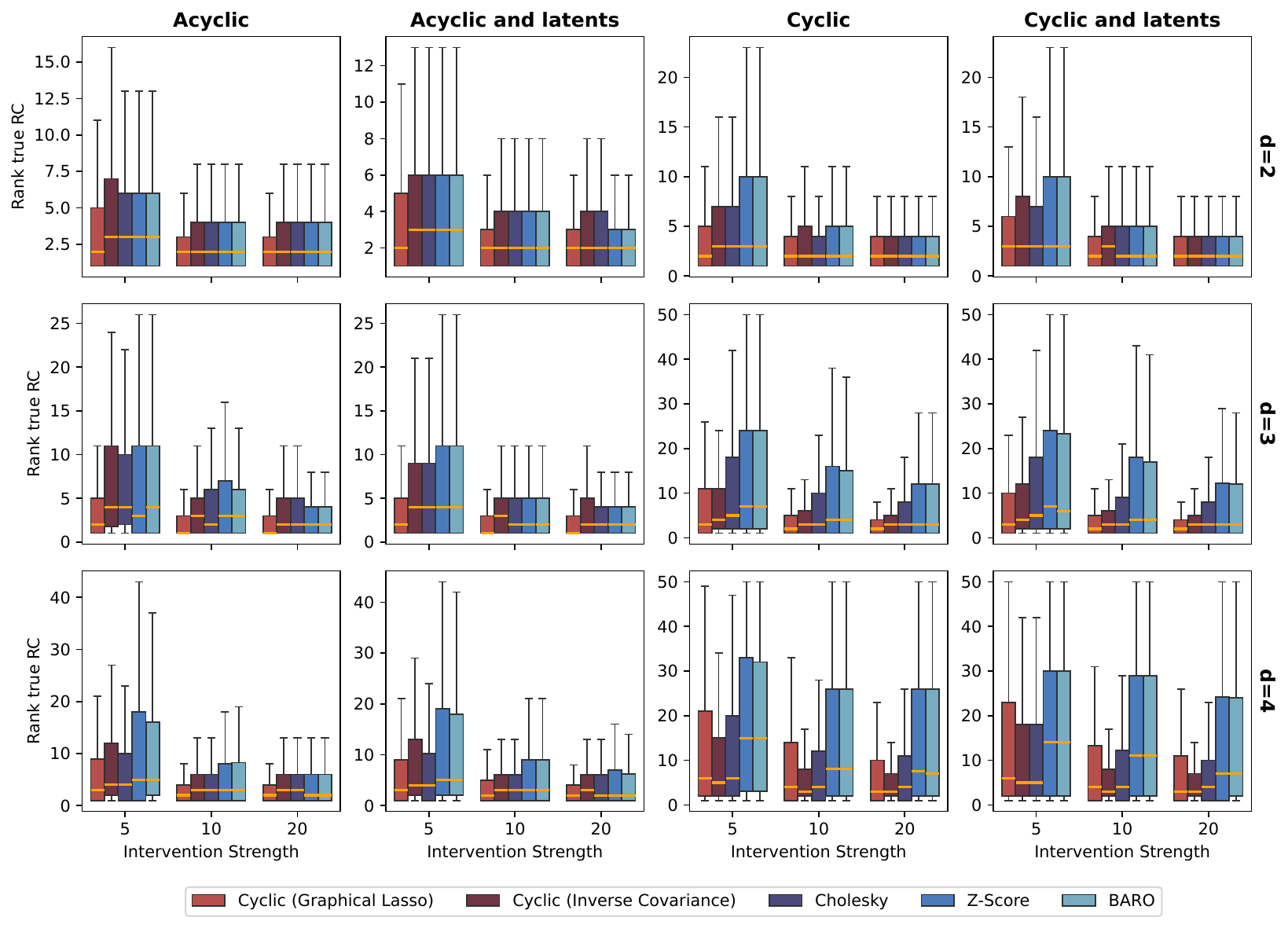}%\\ \includegraphics[width=0.9\linewidth]{vary_noise_distribution_p_100.pdf}
        \caption{Performance at $p = 50$ across varying expected node degrees. All methods exhibit decreased performance as edge degree increases. This decline is minimal for Cyclic (Graphical Lasso), slightly more pronounced for Cyclic (Inverse Covariance) and Cholesky, and markedly strong for BARO and  Z-Score.}
\label{fig:varying_edge_degree_50}
\end{figure}

\subsection{More details on Real-World Experiment}\label{subsec:preprocessing}
\paragraph{Preprocessing}
The PetShop dataset \citep{hardt2024petshop} is organized with a three-level feature hierarchy in each CSV file: the top level corresponds to service names, the second to three core metrics (error rate, availability, and latency) and the third to statistical summaries used to aggregate these metrics over five-minute intervals (e.g., average, sum, and various percentiles). To reduce the data to the service level, we follow the PetShop package's recommendation to select only one relevant metric per incident, based on the nature of the incident. Furthermore, since the ``average'' statistic is the only one consistently available across all services, we restrict to this statistic.

After this filtering, we remove any features that are either constant or exact duplicates of others, as these provide no additional information. We then discard columns with a high proportion of missing values and afterwards remove any rows that still contain NaNs. To normalize the features, we apply a robust standardization, which scales both normal and anomalous data using the mean and standard deviation computed solely from the normal operation period. 

Finally, we apply service-level dimensionality reduction using Principal Component Analysis (PCA).
This step identifies and merges highly correlated services, retaining the principal components that explain at least 90\% of the variance. Doing so, we mitigate the presence of highly correlated features, which would otherwise lead to numerical instability in estimating the precision.

For the remaining real-world datasets, there are no missing entries to fill, and also the PCA step is not necessary. Instead, it suffices to drop almost constant columns and standardize the data.

\paragraph*{Further results on RCAEval}
As mentioned earlier, the benchmark study \cite{pham2024root} evaluates 21 methods on the RCAEval benchmark dataset. In contrast to our experimental section, they do not report the rank of the true root cause. Instead, they report the Avg@5 score, which averages the five metrics counting how often the true root cause appears among the top 1, 2, 3, 4, or 5 ranked nodes. Furthermore, the results are broken down by fault type. In the anomalous datasets, different services are subjected to interventions, and each service is perturbed using one of several fault types: high CPU load, memory leak, disk I/O stress, network delay, or packet loss. Replicating this evaluation setup, we obtain the results reported in Table~\ref{tab:RCAEval_t_0}. In addition to evaluating methods on the original normal and anomalous datasets, the study also considers the scenario in which an incident is detected too late. In this case, a method may be applied under the assumption that the first portion of already anomalous data — for example the first minute — is still part of the normal dataset. The results for this scenario are shown in Table~\ref{tab:RCAEval_t_60}.

\begin{table}[h]
\centering
\caption{AVG@5 scores for each fault type are shown. CYC (G), (S), (I) refers to the Cyclic method with Graphical Lasso, Inverse Covariance, and Shrunk Covariance, respectively. Cholesky is omitted for Sock Shop 2 and Train Ticket due to excessive runtime. For Train Ticket, precision matrices are estimated only with the Shrunk Covariance estimator since the number of features exceeds the number of samples. The synthetic dataset does not include different fault types. Using the AVG@5 metric, the results are still similar to those shown in Figure~\ref{fig:real_world}, which reports the average rank. Among the fault types, the LOSS fault appears to be more challenging for all methods.}
\label{tab:RCAEval_t_0}
\begin{tabular}{llrrrrrr}
\setlength{\tabcolsep}{0pt} 
\textbf{DATASET} & \textbf{FAULT} & \textbf{BARO} & \textbf{CHOLESKY} & \textbf{CYC (G)} &\textbf{ CYC (I)} & \textbf{CYC (S)} & \textbf{Z-SCORE} \\ \hline \\
CIRCA 10 & SIM & 0.82 & 0.44 & 0.84 & 0.79 & 0.85 & 0.81 \\[5pt]
CIRCA 50 & SIM & 0.81 & 0.57 & 0.88 & 0.74 & 0.92 & 0.82 \\[5pt]
RCD 10 & SIM & 0.05 & 0.69 & 0.19 & 0.19 & 0.19 & 0.23 \\[5pt]
RCD 50 & SIM & 0.04 & 0.40 & 0.08 & 0.08 & 0.09 & 0.08 \\[5pt]
\multirow[t]{2}{*}{Sock Shop 1} & CPU & 0.82 & 0.78 & 0.70 & 0.71 & 0.78 & 0.72 \\
 & MEM & 0.97 & 0.86 & 0.91 & 0.88 & 0.90 & 0.96 \\[5pt]
\multirow[t]{5}{*}{Sock Shop 2} & CPU & 0.83 & - & 0.74 & 0.82 & 0.75 & 0.86 \\
 & DELAY & 0.86 & - & 0.88 & 0.89 & 0.89 & 0.90 \\
 & DISK & 0.72 & - & 0.76 & 0.62 & 0.75 & 0.83 \\
 & LOSS & 0.78 & - & 0.77 & 0.68 & 0.81 & 0.82 \\
 & MEM & 0.86 & - & 0.88 & 0.84 & 0.92 & 0.94 \\[5pt]
\multirow[t]{5}{*}{Online Boutique} & CPU & 0.77 & 0.96 & 0.90 & 0.86 & 0.89 & 0.90 \\
 & DELAY & 0.90 & 0.90 & 0.91 & 0.91 & 0.90 & 0.91 \\
 & DISK & 0.92 & 0.82 & 0.83 & 0.82 & 0.86 & 0.90 \\
 & LOSS & 0.57 & 0.56 & 0.54 & 0.46 & 0.51 & 0.65 \\
 & MEM & 0.99 & 1.00 & 1.00 & 1.00 & 1.00 & 0.99 \\[5pt]
\multirow[t]{5}{*}{Train Ticket} & CPU & 0.44 & - & - & - & 0.47 & 0.52 \\
 & DELAY & 0.51 & - & - & - & 0.73 & 0.74 \\
 & DISK & 0.90 & - & - & - & 1.00 & 1.00 \\
 & LOSS & 0.35 & - & - & - & 0.50 & 0.50 \\
 & MEM & 0.70 & - & - & - & 0.85 & 0.87 \\
\end{tabular}
\end{table}

\begin{table}
\centering
\caption{AVG@5 scores when the failure occurrence time is misspecified by 60 seconds, meaning the first 60 anomalous samples are incorrectly included in the normal data. On CIRCA, the performance of all methods drops sharply. In contrast, on the remaining datasets, BARO remains robust to this misspecification, while the performance of the other methods show a moderate decline.}\label{tab:RCAEval_t_60}
\begin{tabular}{llrrrrrr}
\setlength{\tabcolsep}{0pt} 
\textbf{DATASET} & \textbf{FAULT} & \textbf{BARO} & \textbf{CHOLESKY} & \textbf{CYC (G)} &\textbf{ CYC (I)} & \textbf{CYC (S)} & \textbf{Z-SCORE} \\ \hline \\
CIRCA 10 & SIM & 0.28 & 0.34 & 0.13 & 0.18 & 0.10 & 0.17 \\[5pt]
CIRCA 50 & SIM & 0.06 & 0.05 & 0.01 & 0.01 & 0.01 & 0.00 \\[5pt]
RCD 10 & SIM & 0.04 & 0.67 & 0.18 & 0.19 & 0.19 & 0.23 \\[5pt]
RCD 50 & SIM & 0.05 & 0.41 & 0.07 & 0.09 & 0.08 & 0.09 \\[5pt]
\multirow[t]{2}{*}{Sock Shop 1} & CPU & 0.82 & 0.71 & 0.78 & 0.73 & 0.75 & 0.24 \\
 & MEM & 0.94 & 0.70 & 0.70 & 0.61 & 0.69 & 0.34 \\[5pt]
\multirow[t]{5}{*}{Sock Shop 2} & CPU & 0.76 & - & 0.35 & 0.38 & 0.30 & 0.41 \\
 & DELAY & 0.87 & - & 0.55 & 0.48 & 0.48 & 0.53 \\
 & DISK & 0.66 & - & 0.38 & 0.38 & 0.40 & 0.32 \\
 & LOSS & 0.78 & - & 0.51 & 0.50 & 0.43 & 0.48 \\
 & MEM & 0.85 & - & 0.34 & 0.52 & 0.36 & 0.30 \\[5pt]
\multirow[t]{5}{*}{Online Boutique} & CPU & 0.78 & 0.89 & 0.34 & 0.37 & 0.42 & 0.57 \\
 & DELAY & 0.98 & 0.66 & 0.59 & 0.62 & 0.56 & 0.56 \\
 & DISK & 0.89 & 0.68 & 0.50 & 0.50 & 0.45 & 0.49 \\
 & LOSS & 0.54 & 0.53 & 0.41 & 0.39 & 0.48 & 0.46 \\
 & MEM & 0.98 & 0.79 & 0.55 & 0.53 & 0.55 & 0.45 \\[5pt]
\multirow[t]{5}{*}{Train Ticket} & CPU & 0.41 & - & - & - & 0.02 & 0.01 \\
 & DELAY & 0.50 & - & - & - & 0.10 & 0.10 \\
 & DISK & 0.89 & - & - & - & 0.01 & 0.00 \\
 & LOSS & 0.36 & - & - & - & 0.22 & 0.22 \\
 & MEM & 0.70 & - & - & - & 0.05 & 0.04 \\[5pt]
\end{tabular}
\end{table}

Moreover, since the benchmark contains interventions on different services, the data naturally provide three distinct environments. This allows us to apply the backShift method proposed by \cite{Rothenhaeulser2015}  to the  datasets.
More specifically, the data are organized into tuples consisting of one normal dataset and one anomalous dataset. Each tuple corresponds to a specific fault type and a particular service on which the intervention was performed. For every such combination, five replications are available, resulting in a total of $5^3$ dataset tuples.
Previously, we evaluated the algorithms separately on each dataset pair and averaged performance across all experiments. Alternatively, multiple interventions can be aggregated into a single dataset. In particular, by fixing the fault type and replication index and combining the data from the five intervention targets, we obtain a dataset suitable for applying backShift. The resulting performance is reported in Table~\ref{tab:avg5_results_backshift}.
\begin{table}[h]
\centering
\begin{tabular}{lrrrr}
\textbf{FAULT} &  \textbf{SOCK SHOP 1} & \textbf{SOCK SHOP 2} & \textbf{TRAIN TICKET} & \textbf{ONLINE BOUTIQUE} \\  \hline \\
CPU   & 0.44 & 0.40 & 0.29 & 0.35 \\ 
MEM   & 0.44 & 0.45 & 0.26 & 0.33 \\ 
DISK  & -    & 0.42 & 0.27 & 0.27 \\
DELAY & -    & 0.45 & 0.29 & 0.29 \\ 
LOSS  & -    & 0.48 & 0.29 & 0.26 \\ 
\end{tabular}
\caption{Avg@5 results for backShift. Despite being exposed to five times as many samples, the algorithm performs significantly worse than the top RCA approaches in \cite{pham2024root} and our method.}
\label{tab:avg5_results_backshift}
\end{table}

\subsection{Nonlinear data}
Finally, we evaluate our approach using a semisynthetic dataset derived from a \hyperlink{https://www.pywhy.org/dowhy/main/example_notebooks/gcm_rca_microservice_architecture.html}{case study in the DoWhy package} \citep{dowhy, dowhy2}. The dataset is produced by a structural equation model featuring primarily linear but also some nonlinear relationships, designed to resemble patterns typical of microservice data. More specifically,  
\begin{figure}
    \centering
    \includegraphics[width=0.5\linewidth]{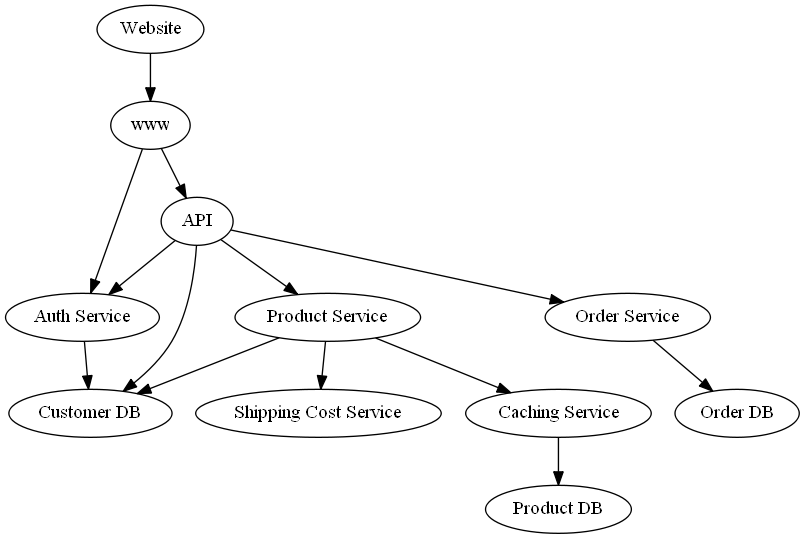}
    \caption{Microservice architecture underlying the semi-synthetic dataset. The graph represents a call graph, where an arrow $A \to B$ indicates that service $A$ calls service $B$. Consequently, in the causal graph, the direction of the arrows are reversed. Figure taken from \href{https://www.pywhy.org/dowhy/main/example_notebooks/gcm_rca_microservice_architecture.html}{DoWhy documentation}. 
    }
    \label{fig:graph_semisynthetic}
\end{figure}
\begin{figure}[!htb]
    \centering
    \includegraphics[width=0.9\linewidth]{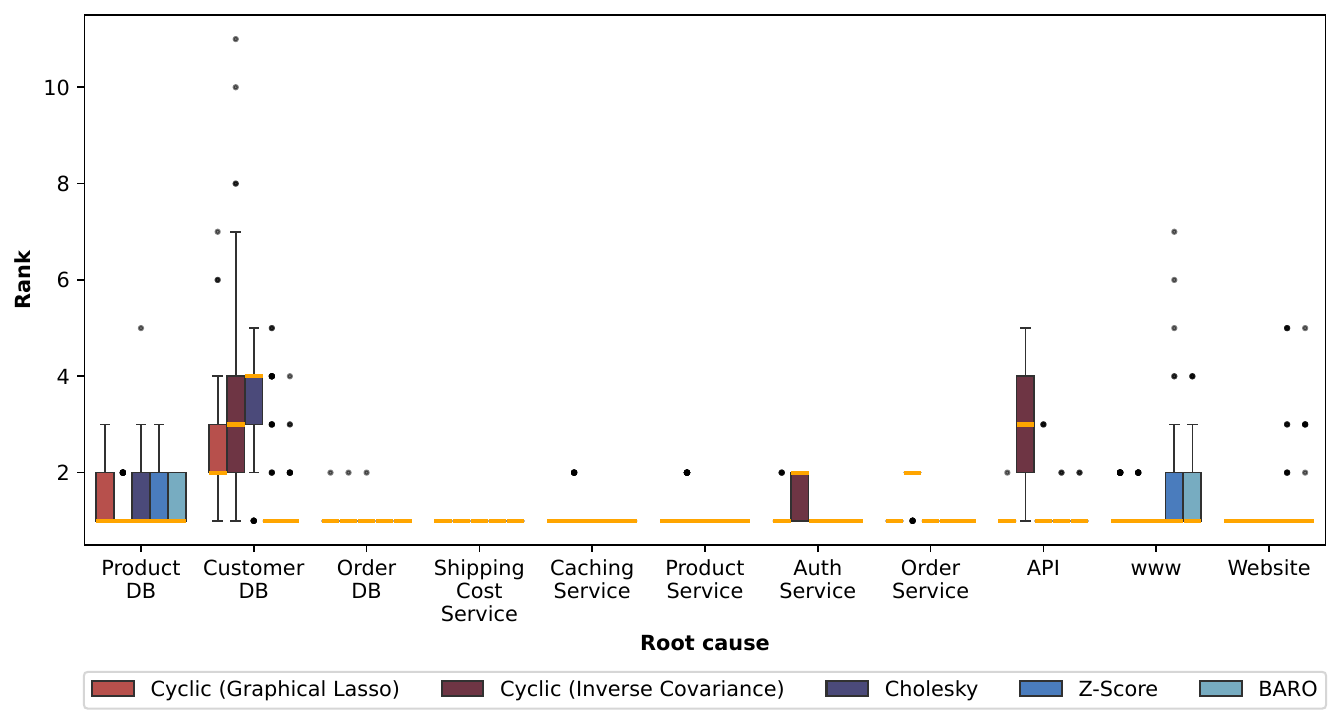}
    \caption{Root cause ranks across 100 replications, with the x-axis indicating the target node. Most methods place the true root cause in the top two, except for  Cyclic (Inverse Covariance) and Cholesky on Customer DB and Cyclic (Inverse Covariance) on API.}
    \label{fig:results_semisynthetic}
\end{figure}
the data is generated according to the graph shown in Figure \ref{fig:graph_semisynthetic} and  each $X_i$ is the sum of its parents plus a noise term, with the following two exceptions:
\begin{align*}
X_{\text{Caching Service}} &= Z \cdot X_{\text{Product DB}} + \epsilon_{\text{Caching Service}}, \text{ with }
Z \sim \text{Bernoulli}(p=0.5); \\
X_{\text{Product Service}} &= \max\left( X_{\text{Shipping Cost Service}},\; X_{\text{Caching Service}},\; X_{\text{Customer DB}} \right) + \epsilon_{\text{Product Service}}
\end{align*}
For the noise distributions, truncated exponential noise (bounded above) is used for \textit{Website}, \textit{www}, \textit{Order DB}, \textit{Customer DB}, and \textit{Product DB}, with varying truncation levels and a common scale parameter of 0.2. Half-normal noise (non-negative) is used for \textit{API}, \textit{Auth Service}, \textit{Product Service}, \textit{Order Service}, \textit{Shipping Cost Service}, and \textit{Caching Service}, with varying means and standard deviations.

In the original case study, the root cause is fixed as the \textit{Caching Service}, perturbed by a shift of 2. By contrast, we extend this setup by conducting a series of experiments — one for every node as the root cause, which we also model by a shift intervention of 2. The RCA methods are then applied to the resulting data. Figure \ref{fig:results_semisynthetic} shows that most methods rank the true root cause among the top two nodes, except for  Cyclic (Inverse Covariance) and Cholesky performing worse when \textit{Customer DB} is the root cause, and the Cyclic (Inverse Covariance) method drops in performance when intervening on \textit{API}.

\end{document}